\documentclass[letterpaper, 10pt, conference]{ieeeconf}      
 
\IEEEoverridecommandlockouts                              

\usepackage{color}
\usepackage{graphicx}
\usepackage{amsmath}
\usepackage{amssymb}
\usepackage{tabularx}
\usepackage{subcaption}
\usepackage{url}
\usepackage{cite}
\usepackage{algorithm,algorithmicx,algpseudocode}
\usepackage{blkarray}

\newtheorem{theorem}{Theorem}[section]

\newtheorem{lemma}[theorem]{Lemma}
\newtheorem{prop}{Proposition}
\algnewcommand{\LineComment}[1]{\State \(\triangleright\) #1}
\newcommand{\TWOC}[2]{\left(\setlength{\arraycolsep}{1pt}\begin{array}{c}#1 \\ #2\end{array}\right)}
\newcommand{\TWORCell}[2]{\begin{tabular}{@{}l@{}}#1 \\ #2\end{tabular}}

\DeclareMathOperator*{\argmin}{\arg\!\min}
\DeclareMathOperator*{\argmax}{\arg\!\max}

\title{\LARGE  \bf Learning-based Feedback Controller for Deformable Object Manipulation} 

\author{Biao Jia*, Zhe Hu*, Zherong Pan, Dinesh Manocha, Jia Pan \\%
\url{https://sites.google.com/view/lfcdom/}
\thanks{Biao Jia, Zherong Pan, and Dinesh Manocha are with the Department of Computer Science, The University of North Carolina, Zhe Hu and Jia Pan are with the Department of Mechanical and Biomedical Engineering, City University of Hong Kong. }%
}

\begin{document}
\maketitle

\begin{abstract}
In this paper, we present a general learning-based framework to automatically visual-servo control the position and shape of a deformable object with unknown deformation parameters.
The servo-control is accomplished by learning a feedback controller that determines the robotic end-effector's movement according to the deformable object's current status. 
This status encodes the object's deformation behavior by using a set of observed visual features, which are either manually designed or automatically extracted from the robot's sensor stream.
A feedback control policy is then optimized to push the object toward a desired featured status efficiently. The feedback policy can be learned either online or offline. Our online policy learning is based on the Gaussian Process Regression (GPR), which can achieve fast and accurate manipulation and is robust to small perturbations. An offline imitation learning framework is also proposed to achieve a control policy that is robust to large perturbations in the human-robot interaction.
We validate the performance of our controller on a set of deformable object manipulation tasks and demonstrate that our method can achieve effective and accurate servo-control for general deformable objects with a wide variety of goal settings. 
\end{abstract}

\section{Introduction}

Robot manipulation has been extensively studied for decades and there is a large body of work on the manipulation of rigid and deformable objects. Compared to the manipulation of a rigid object, the state of which can be completely described by a six-dimensional configuration space, deformable object manipulation (DOM) is more challenging due to its very high configuration space dimensionality. The resulting manipulation algorithm needs to handle this dimensional complexity and maintain the tension to perform the task. DOM has many important applications, including cloth folding~\cite{Miller:2011:AGA, Towner:2011:BCI};
robot-assisted dressing or household chores~\cite{Kapusta:2016:DDH, Gao:2016:IPO}; ironing~\cite{Li:2016:MSS}; coat checking~\cite{Twardon:2015:ISC}; sewing~\cite{Schrimpf:2012:RSI};
string insertion~\cite{Wang:2015:AOM}; robot-assisted surgery and suturing~\cite{Patil:2010:TAT,Schulman:2013:GIR}; and transporting large materials like cloth, leather, and composite materials~\cite{Kruse:2015:CHR}.

\begin{figure}[t] 
\centering
\includegraphics[width=0.9\linewidth]{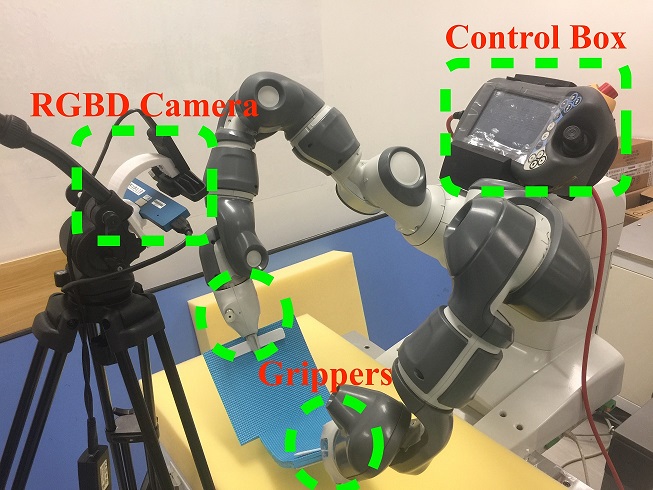}
\caption{Our robotic system for deformable object manipulation is made of two 3D cameras and one dual-arm ABB robot. }
\label{fig:manipulation-system}
\end{figure}

There are two main challenges that arise in performing DOM tasks. 
First, we need to model a feature representation of the object status that can account for the object's possibly large deformations. The design of such a feature is usually task-dependent and object-specific. For instance, a cloth needs richer features than a rope, and a surgical-related application requires more sophisticated features than a consumer application. As a result, feature design is usually a tedious manual procedure. Second, we need to develop a controller that maps the object's state features to a robot control action in order to achieve a specific task. The design of such a controller is usually also task-dependent. For instance, a manipulation controller for human dressing will require a more expressive controller parametrization than that needed to flatten a piece of paper. As a result, controller design usually requires tedious manual parameter tuning when generalizing to different tasks. More importantly, the two problems of feature extraction and controller design were solved as two separate problems. With the recent development of deep (reinforcement) learning, a prominent method \cite{Levine:2016:ETD,mnih:2015:HLC} is to represent feature extraction and controller parametrization as two neural networks, which are either trained either jointly or separately. However, the design decisions for these two networks, e.g., their network structures, are made independently. This method suffers from the large number of parameters that need to be manually determined. There are some more task-specific methods, such as \cite{Li:2015:FDO,Li:2016:MSS}, which use a vision-based feature extractor, but the controller is a standalone optimization-based motion planner whose formulation is independent of feature extraction. 
These methods are hard to extend because the feature-based objective functions for trajectory optimization are task-specific. 

There is a rich literature on autonomous robotic DOM finding solutions to the above challenges, and these works can be classified into three categories. The first group of approaches requires a physical model of the object's deformation properties in terms of stiffness, Young's modules, or FEM coefficients, to design a control policy~\cite{McConachie:2016:BBM, Patil:2010:TAT, Bodenhagen:2014:AAR,Alarcon:2013:MFV,Alarcon:2016:ADM,Langsfeld:2016:OLP}. However, such deformation parameters are difficult to estimate accurately and may even change during the manipulation process, especially for objects made by nonlinear elastic or plastic materials. The approaches in the second group use manually designed low-dimensional features to model the object's deformation behavior~\cite{Sullivan:1996:UAD,Alarcon:2016:ADM,Alarcon:2013:MFV,Alarcon:2018:MFV} and then use an adaptive linear controller to accomplish the task. The approaches in the final group do not explicitly model the deformation parameters of the object. Instead, they use vision or learning methods to accomplish tasks directly~\cite{Miller:2011:AGA,Schulman:2013:GIR,Hadfield:2015:BLW,Tang:2016:RMD,Yang:2017:RFT,Li:2015:FDO,Doumanoglou:2016:FCA,Tanaka:2018:EMDNet}. These methods focus on high-level policies but cannot achieve accurate operations; their success rate is low and some of them are open-loop methods. As a result, designing appropriate features and controllers to achieve accurate and flexible deformable object manipulation is still an open research problem in robotics~\cite{Henrich:2012:RMO}.

In this paper, we focus on designing a general learning-based feedback control framework for accurate and flexible DOM. The feedback controller's input is a feature representation for the deformable object's current status and its output is the robotic end-effector's movement. Our framework provides solutions to both the feature design and controller parameterization challenges. For feature design, we propose both a set of manually designed low-level features and a novel higher-level feature based on a histogram of oriented wrinkles, which is automatically extracted from data to describe the shape variation of a highly deformable object. For controller parameterization, we first propose a novel nonlinear feedback controller based on Gaussian Process Regression (GPR), which learns the object's deformation behavior online and can accomplish a set of challenging DOM tasks accurately and reliably. We further design a controller that treats feature extraction and controller design as a coupled problem by using a random forest trained using two-stage learning. During the first stage, we construct a random forest to classify a sampled dataset of images of deformable objects. Based on the given forest topology, we augment the random forest by defining one optimal control action for each leaf-node, which provides the action prediction for any unseen image that falls in that leaf-node. In this way, the feature extraction helps determine the parameterization of the controller. The random forest construction and controller optimization are finally integrated into an imitation learning framework to improve the robustness of human-robot co-manipulation tasks.

We have integrated our approach with an ABB YuMi dual-arm robot and a camera for image capture and use this system to manipulate different cloth materials for different tasks. We highlight the real-time performance of our method on a set of DOM benchmarks, including standard feature point reaching tasks (as in~\cite{Alarcon:2013:MFV,Alarcon:2016:ADM}); the tasks for cloth stretching, folding, twisting, and placement; and the industrial tasks for cloth assembly.  Our manipulation system successfully and efficiently accomplishes these manipulation tasks for a wide variety of objects with different deformation properties. 

Our main contributions in this paper include:
\begin{itemize}
\item A general feedback control framework for DOM tasks.
\item A set of manually designed low-level features that work well in a set of challenging DOM tasks.
\item A novel histogram feature representation of highly deformable materials (HOW-features) that are computed directly from the streaming RGB data using Gabor filters. These features are then correlated using a sparse representation framework with a visual feedback dictionary and then fed to the feedback controller to generate appropriate robot actions.
\item An online Gaussian Process Regression based nonlinear feedback controller for DOM tasks that is robust and adaptive to the object's unknown deformation parameters.
\item A random-forest-based DOM-controller parametrization that is robust and less parameter-sensitive; an imitation learning algorithm based framework that trains robust DOM controllers using a deformable object simulator.
\item A set of benchmark DOM tasks that have importance for manufacturing or service industries.
\end{itemize}

The techniques developed in this paper would be useful for general DOM tasks, including the DOM tasks in the warehouse. In addition, our techniques could potentially be used in some manufacturing processes. Let's take the assembly of cloth pieces with fixtures, one typical deformable object manipulation task that we will study in this paper, as an example. In this task, a piece of cloth with holes needs to be aligned with a fixture made by a set of vertical locating pins. The assembled cloth pieces are then sent to the sewing machine for sewing operations. Such a task can be efficiently performed by a human worker without any training, as shown in Figure~\ref{fig:worker}, but it is difficult for a robot. For instance, the wrinkles generated during the operation will interfere the feature extraction and tracking procedures that are critical for perception feedbacks. And the highly deformable property of the cloth will lead to unpredictable changes in the size and shape of the holes during the manipulation. These challenges make it difficult to achieve an accurate and reliable robotic manipulation control for cloth assembly. The methods developed in this paper would help to enable the robot to accomplish the cloth assembly tasks accurately and efficiently. 

\begin{figure}[t] 
\centering
\begin{subfigure}{0.15\textwidth}
\includegraphics[width=1.0\linewidth, height=3.5cm]{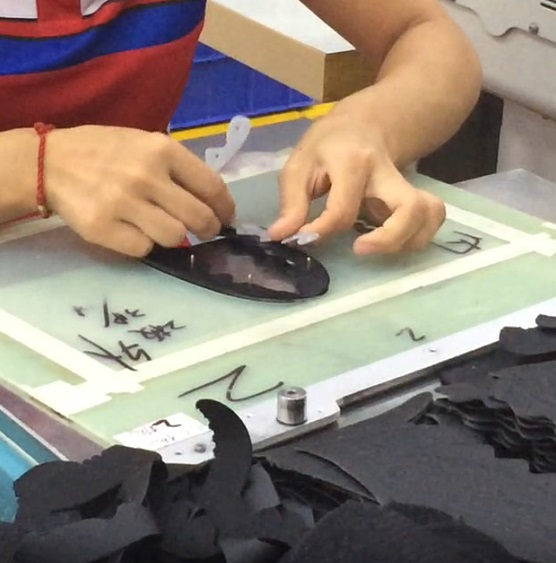}
\caption{}
\end{subfigure}
\begin{subfigure}{0.15\textwidth}
\includegraphics[width=1.0\linewidth, height=3.5cm]{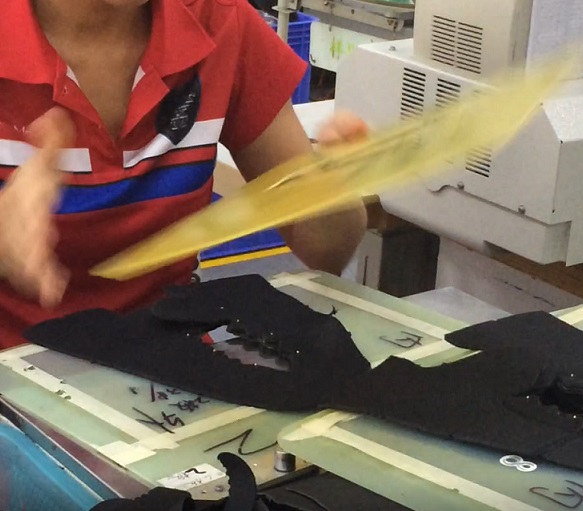}
\caption{}
\end{subfigure}
\begin{subfigure}{0.15\textwidth}
\includegraphics[width=1.0\linewidth, height=3.5cm]{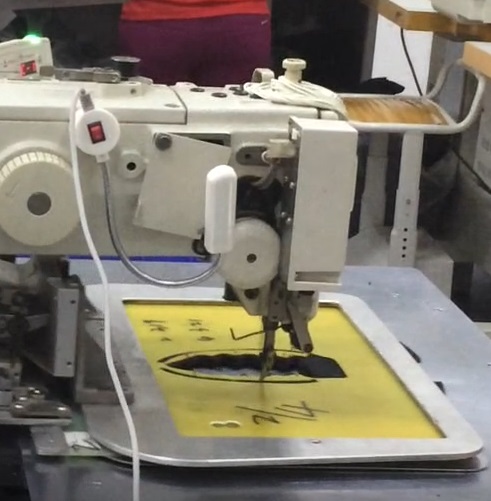}
\caption{}
\end{subfigure}
\caption{The manual procedure in a clothing manufacturer where a human worker is using fixtures to assemble cloth pieces for the following automated sewing: (a) a human worker locates the cloth pieces on the fixtures made by a few pins; (b) the human worker finishes the assembly of cloth pieces; (c) an industrial sewing machine is performing the automated sewing over the assembled template.  Our DOM framework can enable the robot to accomplish this manufacturing process autonomously. }
\label{fig:worker}
\end{figure}

The rest of this paper is organized as follows. We briefly survey the related works in Section~\ref{sec:related}. We give an overview of our DOM framework in Section~\ref{sec:prob}. We present the details of our feature design and extraction in Section~\ref{sec:featuredesign}. We discuss the details of our new controllers in Section~\ref{sec:controllerdesign}. Finally, we demonstrate the performance of our new approach with a set of experiments on a wide variety of soft objects in Section~\ref{sec:experiment} with conclusions in Section~\ref{sec:conclusion}.


\section{Related work}
\label{sec:related}

\subsection{Deformable object manipulation}

Many robotic manipulation methods for deformable objects have been proposed in recent years. Early work~\cite{Saha:2007:MPD,Moll:2006:PPD} used knot theory or energy theory to plan the manipulation trajectories for linear deformable objects like ropes. Some recent work~\cite{Bai:2016:DMC} considered manipulating clothes using dexterous grippers. These works required a complete and accurate knowledge about the object's geometric and deformation parameters and thus are not applicable in practice. 

More practical works used sensors to guide the manipulation process. ~\cite{Matsuno:2006:MDL} used images to estimate the knot configuration. ~\cite{Miller:2011:AGA} used vision to estimate the configuration of a cloth and then leveraged gravity to accomplish folding tasks~\cite{Bell:2010:GNS}. ~\cite{Twardon:2015:ISC} used an RGBD camera to identify the boundary components in clothes. ~\cite{Li:2015:FDO,Li:2015:RUG} first used vision to determine the status of the cloth, then optimized a set of grasp points to unfold the clothes on the table, and finally found a sequence of folding actions.  
Schulman \emph{et al.}~\cite{Schulman:2013:GIR} enabled a robot to accomplish complex multi-step deformation object manipulation strategies by learning from a set of manipulation sequences with depth images to encode the task status. Such learning from demonstration techniques have further been extended using reinforcement learning~\cite{Hadfield:2015:BLW} and tangent space mapping~\cite{Tang:2016:RMD}. 
A deep learning-based end-to-end framework has also been proposed recently~\cite{Yang:2017:RFT}. A complete pipeline for clothes folding tasks including vision-based garment grasping, clothes classification and unfolding, model matching and folding has been described in~\cite{Doumanoglou:2016:FCA}. 

The above methods generally did not explicitly model the deformation parameters of the deformation objects, which is necessary for high-quality manipulation control. Some methods used uncertainty models~\cite{McConachie:2016:BBM} or heuristics~\cite{Berenson:2013:MDO, Wang:2015:AOM} to account for rough deformation models during the manipulation process. Some works required an offline procedure to estimate the deformation parameters~\cite{Bodenhagen:2014:AAR}. There are several recent works that estimated the object's deformation parameters in an online manner and then designed a controller accordingly. Navarro-Alarcon \emph{et al.}~\cite{Alarcon:2013:MFV,Alarcon:2016:ADM,Alarcon:2018:MFV} used an adaptive and model-free linear controller to servo-control soft objects, where the object's deformation is modeled using a spring model~\cite{Hirai:2000:ISP}. \cite{Langsfeld:2016:OLP} learned the models of the part deformation depending on the end-effector force and grasping parameters in an online manner to accomplish high-quality cleaning tasks. A more complete survey about deformable object manipulation in industry is available in\cite{Henrich:2012:RMO}.

\subsection{Deformable object feature design}
Different techniques have been proposed for motion planning for deformable objects. 
Most of the works on deformable object manipulation focus on volumetric objects such as a deforming ball or linear deformable objects such as steerable needles~\cite{Rodriguez:2006:PMC,Frank:2011:EMP,Reed:2011:RAN,Saha:2007:MPD}. 
By comparison, cloth-like thin-shell objects tend to exhibit more complex deformations, forming wrinkles and folds. 
Current solutions for thin-shelled manipulation problems are limited to specific tasks, including folding \cite{Li:2015:FDO}, ironing \cite{Li:2016:MSS}, sewing \cite{Schrimpf:2012:RSI}, and dressing \cite{Clegg:2015:AHD}. 
On the other hand, deformable body tracking solves a simpler problem, namely inferring the 3D configuration of a deformable object from sensing inputs. 
There is a body of literature on deformable body tracking that deals with inferring the 3D configuration from sensor data \cite{Schulman:2013:TDO,Wang:2015:DCM,Leizea:2017:RTV}.
However, these methods usually require a template mesh to be known a priori, and the allowed deformations are relatively small. Recently, some template-less approaches have also been proposed,  including~\cite{Newcombe:2015:dynamicfusion,innmann:2016:volumedeform,dou:2016:fusion4d}, that tackle the tracking and reconstruction problems jointly and in real-time,

Rather than requiring a complete 3D reconstruction of the entire object, a visual-servo controller only uses single-view observations about the object as the input. However, even the single-view observation is high-dimensional. Thus, previous DOM methods use various feature extraction and dimensionality-reduction techniques, including SIFT-features~\cite{Li:2015:FDO} and combined depth and curvature-based features~\cite{Doumanoglou:2016:FCA,Doumanoglou:2014:AAR}. Recently, deep neural-networks became a mainstream general-purpose feature extractor. 
They have also been used for manipulating low-DOF articulated bodies \cite{Levine:2016:ETD} and for DOM applications~\cite{yang2017repeatable}. However, visual feature extraction is always decoupled from controller design in all these methods.

\subsection{Deformable object controller optimization}
In robotics, reinforcement learning~\cite{Sutton:1998:IRL}, imitation learning ~\cite{Hussein:2017:ILS}, and trajectory optimization~\cite{Stengel:1986:SOC} 
have been used to compute optimal control actions.
Trajectory optimization, or a model-based controller, has been used in ~\cite{Li:2015:FDO,Li:2016:MSS,Lee:2014:USR} for DOM applications. Although they are accurate, these methods cannot achieve real-time performance. 
For low-DOF robots such as articulated bodies~\cite{Todorov:2012:MUJOCO}, researchers have developed real-time trajectory optimization approaches, but it is hard to extend them to deformable models due to the high simulation complexity of such models.
Currently, real-time performance can only be achieved by learning-based controllers \cite{Doumanoglou:2014:ARF,Doumanoglou:2014:AAR,yang2017repeatable}, which use supervised learning to train real-time controllers. However, as pointed out in~\cite{Ross:2011:RIL}, these methods are not robust in handling unseen data. Therefore, we can further improve the robustness by using imitation learning. 
~\cite{Gupta:2016:LDM} used reinforcement learning to control a soft-hand, but the object to be manipulated by the soft-hand was still rigid.

DOM controller design is dominated by visual servoing techniques~\cite{chaumette:2006:visual,hutchinson:1996:tutorial}, which aim to control a dynamic system using visual features extracted from images. They have been widely used in robotic tasks like manipulation and navigation. Recent work includes the use of histogram features for rigid objects.

The dominant approach for DOM tasks is visual servoing techniques ~\cite{chaumette:2006:visual,hutchinson:1996:tutorial}, which aim at controlling a dynamic system using visual features extracted from images. 
These techniques have been widely used in robotic tasks like manipulation and navigation. 
Recent work includes the use of histogram features for rigid objects~\cite{bateux:2017:histograms}.
Sullivan \emph{et al.}~\cite{Sullivan:1996:UAD} use a visual servoing technique to solve the deformable object manipulation problem using active models. 
Navarro-Alarcon \emph{et al.}~\cite{Alarcon:2013:MFV,Alarcon:2016:ADM,Alarcon:2018:MFV} use an adaptive and model-free linear controller to servo-control soft objects, where the object's deformation is modeled using a spring model~\cite{Hirai:2000:ISP}. 
Langsfeld \emph{et al.}~\cite{Langsfeld:2016:OLP} perform online learning of part-deformation models for robot cleaning of compliant objects. Our goal is to extend these visual servoing methods to perform complex tasks on highly deformable materials.

\section{Overview and problem formulation}
\label{sec:prob}

\begin{table}[t]
\setlength{\tabcolsep}{15pt}
\begin{tabular}{ll}
\hline 
Symbol & Meaning   \\
\hline \hline
$\mathcal{C}$ & 3D configuration space of the object  \\
$c$ & a configuration of the object, \\
 & $c= \{\mathbf p\} \cup \{\mathbf q\} \cup \{\mathbf r\}$   \\
$\mathbf p$ & feedback feature points on the object \\
$\mathbf q$ & uninformative points on the object \\
$\mathbf r$ & robot end-effectors' grasping points on the object \\
$\mathcal{O}(c)$ & an observation of object  \\
$\mathbf s$ & the feature vector extracted from $\mathcal{O}(c)$ \\
$c^*$ & target configuration of the object  \\
$\mathbf r^*$ & optimal grasping points returned by the expert	\\
$H$ & interaction function linking velocities of  the feature \\
& space to the end-effector configuration space	\\
$\mathbf{dist}_s$ & distance measure in the feature space \\
$\mathbf{dist}_{\pi}$ & distance measure in the policy space \\
$\pi$ & DOM-control policy	\\
\hline
& Random forest DOM controller \\
\hline \hline
$\alpha$ & parameter for random forest topology	\\
$\beta$ & leaf parameter	\\
$\gamma$ & confidence of leaf-node	\\
$\theta$ & parameter sparsity	\\
$K$ & the number of decision trees	\\
$l_k$ & a leaf-node of $k$-th decision tree	\\
$l_k(\mathcal{O}(c))$ & the leaf-node that $\mathcal{O}(c)$ belongs to \\
$\mathcal{L}$ & labeling function for optimal actions	\\
$\mathcal{F}$ & feature transformation for observation	\\
\hline
\end{tabular}
\caption{\label{Fig:param} Symbol table.}
\end{table}

The problem of 3D deformable object manipulation can be formulated as follows. Similar to~\cite{Hirai:2000:ISP}, we describe an object as a set of discrete points, which are classified into three types: manipulated points, feedback points, and uninformative points, as shown in Figure~\ref{fig:problem-formulation}. The manipulated points correspond to the positions on the object that are grabbed by the robot and are thus fixed relative to the robotic end-effectors. The feedback points correspond to the object surface regions that define the task goals and are involved in the visual feedbacks. The uninformative points correspond to other regions on the object. Given this setup, the deformable object manipulation problem is about how to move the manipulated points to drive the feedback points toward a required target configuration. 

Our goal is to optimize a realtime feedback controller to deform the object into a desired target configuration. We denote the 3D configuration space of the object as $\mathcal{C}$. Typically, a configuration $c\in\mathcal{C}$ can be discretely represented as a 3D mesh of the object and the dimension of $\mathcal{C}$ can be in the thousands. However, we assume that only a partial observation $\mathcal{O}(c)$ is known, which in our case is an RGB-D image from a single, fixed point of view. 

Since the manipulation of a deformable object is usually executed at a low speed to avoid vibration, we can reasonably assume that the object always lies in the quasi-static state where the internal forces caused by the elasticity of the object are balanced with the external force applied by the end-effector on the manipulated points. We use a potential energy function $U(\mathbf p, \mathbf q, \mathbf r)$ to formulate the elasticity of the deformable object, where the potential energy depends on all the points on the object and vectors $\mathbf p$, $\mathbf q$, and $\mathbf r$ represent the stacked coordinates of all feedback points, uninformed points, and manipulated points, respectively. The equation of equilibrium for the object can then be described as follows:
\begin{align}
\frac{\partial U}{\partial \mathbf p} &= \mathbf 0, \label{eq:equf} \\
\frac{\partial U}{\partial \mathbf q} &= \mathbf 0, \\
\frac{\partial U}{\partial \mathbf r} - \mathbf F &= \mathbf 0, \label{eq:equm}
\end{align}
where $\mathbf F$ is the external force vector applied on the manipulated points. To solve the above equations, we need exact knowledge about that deformable object's deformation property, which is either not available or difficult to acquire in many applications. To cope with this issue, we first simplify the potential energy function to only depend on $\mathbf p$ and $\mathbf r$, which is reasonable because the uninformed points are usually far from the manipulated and feedback points and thus their influence on the manipulation process is small and can be neglected. Next, we perform a Taylor expansion of Equation~\ref{eq:equm} and Equation~\ref{eq:equf} about the current static equilibrium status $(\bar{\mathbf p}, \bar{\mathbf r})$, and the equation of equilibrium implies a relationship between the relative displacements of feedback points and manipulated points:
\begin{equation}
A(\delta \mathbf p) + B(\delta \mathbf r) = \mathbf 0,
\end{equation}
where $\delta \mathbf p = \mathbf p -\bar{\mathbf p}$ and $\delta \mathbf r = \mathbf r- \bar{\mathbf r}$ are the displacements relative to the equilibrium for feedback points and manipulated points, respectively. The functions $A(\cdot)$ and $B(\cdot)$ are nonlinear in general, though they can be linear in some special cases. For instance, when only performing the first order Taylor expansion as in~\cite{Alarcon:2016:ADM}, $A(\delta \mathbf p) = \frac{\partial^2 U}{\partial \mathbf r \partial \mathbf p}$ and $B(\delta \mathbf r) = \frac{\partial^2 U}{\partial (\mathbf r)^2} \delta \mathbf r$ are two linear functions. In this paper, we allow $A(\cdot)$ and $B(\cdot)$ to be general functions to estimate a better model for the deformable object manipulation process. 

We further assume the function $B(\cdot)$ to be invertible, which implies
\begin{equation}
\label{eq:D}
\delta \mathbf r = D(\delta \mathbf p),
\end{equation}
where $D = A \circ B^{-1}$ is the mapping between the velocities of the feedback points and the manipulated points. In this way, we can determine a suitable end-effector velocity via feedback control $\delta_r =  D(\eta \cdot \Delta \mathbf p)$ to derive the object toward its goal state, where $\Delta \mathbf p = \mathbf p^* - \mathbf p$ is the difference between the desired vector $\mathbf p^*$ and the current vector of the feedback points and $\eta$ is the feedback gain. 

\begin{figure}[t] 
\centering
\includegraphics[width=0.9\linewidth]{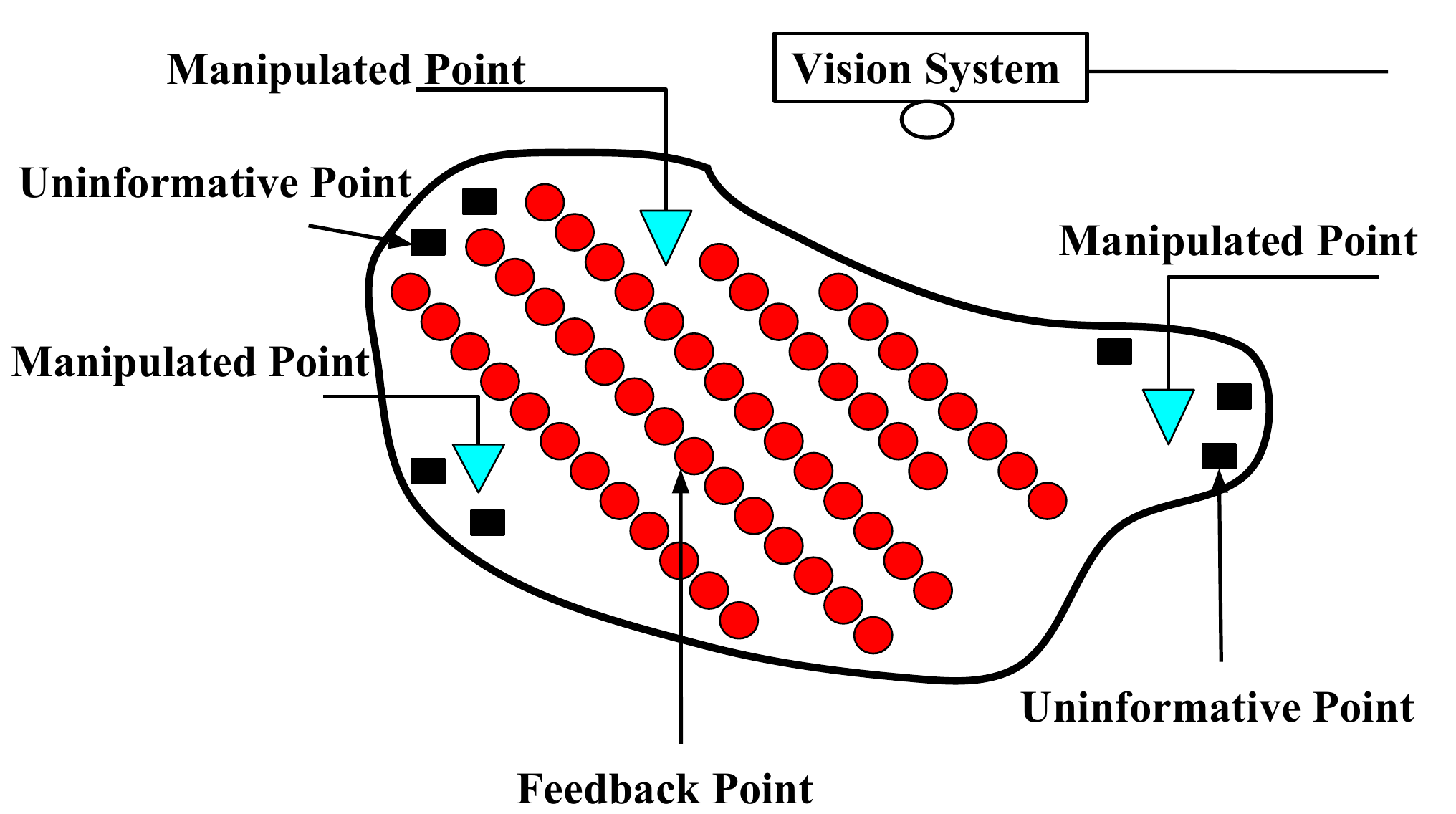}
\caption{We model a soft object using three classes of points: manipulated points $\mathbf r$, feedback points $\mathbf p$, and uninformative points $\mathbf q$.}
\label{fig:problem-formulation}
\end{figure}

However, the velocities of feedback points generally cannot be directly used in the control, because, in practice, these velocities are measured using visual tracking of deformable objects and thus are likely to be noisy or even unreliable when tracking fails. More importantly, a soft object needs a large set of feedback points to characterize its deformation, but a robotic manipulation system usually only has a few end-effectors. Thus ,the $D(\cdot)$ function in Equation~\ref{eq:D} is a mapping from the high-dimensional space of feedback point velocities to the low-dimensional space of manipulated point velocities. Such a system is extremely underactuated and the convergence speed of the control would be slow. 

To deal with aforementioned difficulties, we replace the feedback points $\mathbf p$ with a low-dimensional feature vector $\mathbf s = C(\mathbf p)$, which is extracted from the observed part of the deformable object, where $C(\cdot)$ is the feature extraction function. Around the equilibrium state, we have $\delta \mathbf s = C'(\bar{\mathbf p}) \delta \bar{\mathbf p}$, and we can rewrite the equilibrium function using the feature vector as
\begin{equation}
\label{eq:mapping}
\delta \mathbf r = D(C'(\bar{\mathbf p})^{-1} \delta \mathbf s) \triangleq H(\delta \mathbf s),
\end{equation}
where the function $H(\cdot)$ is called the \textbf{interaction function}. 

The manipulation problem of deformable objects can finally be described as the following: given the desired state $\mathbf s^*$ of an object in the feature space, design a controller that learns the interaction function $H(\cdot)$ in an offline or online manner, and outputs the control velocity $\delta \mathbf r$ decreasing the distance between $\mathbf s$ and $\mathbf s^*$, which are the object's current state and desired goal state in the feature space. More formally, the controller is: 
\begin{equation}
\delta \mathbf r =  H\big(\eta \cdot \textbf{dist}_s (\mathbf s, \mathbf s^*)\big),
\end{equation}
where $\textbf{dist}_s$ is the distance metric in the feature space and $\eta$ is the feedback gain. If we assume $H$ to be a linear function and $\mathbf{dist}_s$ to be a linear metric, then the controller degrades to the linear visual-servo controller: 
\begin{equation}
\label{eq:basicfb}
\delta \mathbf r = \eta \cdot (\mathbf s - \mathbf s^*).
\end{equation}

In addition, since $\delta \mathbf r = \mathbf r - \bar{\mathbf r}$ and $\bar{\mathbf{r}}$ is usually a known vector, we can also write the controller as a policy $\pi$:
\begin{equation}
\mathbf r = H\big(\eta \cdot \textbf{dist}_s (\mathbf s, \mathbf s^*)\big) + \bar{\mathbf r} = \pi\big(\mathbf s|\mathcal{O}(c)\big),
\end{equation}
where $\mathcal O(c)$ is the observation about the current state of the deformable object and is equivalent to $\mathbf s$. This form is more popular in the policy optimization literature.

\section{Feature extraction: manual design}
\label{sec:featuredesign}
For rigid body manipulation, an object's state can be completely described by its centroid position and orientation. However, such global features are not sufficient to determine the configuration of a deformable object. As mentioned in Section~\ref{sec:prob}, we extract a feature vector $\mathbf s$ from the observed feedback points to represent the object's state. One common practice for feature extraction is to manually design task-dependent features. Here we present a set of features that are able to provide high-quality representations for DOM tasks, as will be demonstrated in our experiments.

\subsection{Global features}
\label{sec:globalfeatures}

\subsubsection{Centroid} The centroid feature $\mathbf c \in \mathcal R^3$ is computed as the geometric center of the 3D coordinates of all the observed feedback points:
\begin{equation}
\label{eq:centroid}
	\mathbf c = \left ( \mathbf p_1 + \mathbf p_2 + \cdots + \mathbf p_K \right ) / K.
\end{equation}

\subsubsection{Positions of feedback points} Another way to describe a deformable object's configuration is to directly use the positions of all observed feedback points as part of $\mathbf x$, i.e. 
\begin{equation}
\label{eq:position}
	\mathbf \rho = \left[ (\mathbf p_1)^T, (\mathbf p_2)^T, \cdots, (\mathbf p_K)^T \right]
\end{equation}

\subsection{Local features}
\label{sec:localfeatures}

\subsubsection{Distance between points} The distance between each pair of feedback points intuitively measures the stretch of deformable objects. This feature is computed as 
\begin{equation}
\label{eq:distance}
	d = \| \mathbf p_1 - \mathbf p_2\|^2,
\end{equation}
where $\mathbf p_1$ and $\mathbf p_2$ are a pair of feedback points.

\subsubsection{Surface variation indicator} For deformable objects with developable surfaces, the surface variation around each feedback point can measure the local geometric property. Given a feedback point $\mathbf p$, we can compute the covariance matrix $\Omega$ for its neighborhood and the surface variation $\sigma$ is then computed as
\begin{equation}
\label{eq:variation}
	\sigma = \frac{\lambda_0}{\lambda_0+\lambda_1+\lambda_2}
\end{equation}
where $\lambda_0$, $\lambda_1$, $\lambda_2$ are eigenvectors of $\Omega$ with $\lambda_0\leq\lambda_1\leq\lambda_2$.

\subsubsection{Extended FPFH from VFH} Extended FPFH is the local descriptor of VFH and is based on Fast Point Feature Histograms (FPFH)~\cite{Rusu:2009:FPF}. Its idea is to use a histogram to record differences between the centroid point $\mathbf p_c$ and its normal $\mathbf n_c$ with all other points and normals.

\subsection{Feature selection for different tasks} 
According to our experience, for 1D deformable objects such as a rope, we can use centroid, distance or the coordinates of several marked points; for 2D deformable objects like deformable sheets, a surface variation indicator or extended FPFH are more effective for representing the states of the deformable objects. The features such as centroid and distance are also used in~\cite{Alarcon:2016:ADM}, and we will use these features to compare our method with~\cite{Alarcon:2016:ADM}. Our learning-based controllers (discussed later) are not very sensitive to the feature selection because they can adaptively weight the importance of different features. Nevertheless, we can introduce task-relevant prior knowledge into the controller by manually designing task-specific features, which would be helpful for the robustness and effectiveness of the controlling process.

\section{Feature extraction: data-driven design}
\label{sec:featuredesign2}

Manual feature design is tedious and cannot be optimized for a given task. Here, we present a novel feature representation, a Histogram of Oriented Wrinkles (HOW), to describe the shape variation of a highly deformable object like clothes. 
These features are computed by applying Gabor filters and extracting the high-frequency and low-frequency components.
We precompute a visual feedback dictionary using an offline training phase that stores a mapping between these visual features $\mathbf s$ and the velocity of the end-effector $\delta \mathbf r$. At runtime, we automatically compute the goal configurations based on the manipulation task and use sparse linear representation to compute the velocity of the controller from the dictionary. 

\begin{figure*}[!htbp]
  \centering
  \includegraphics[width=1\textwidth]{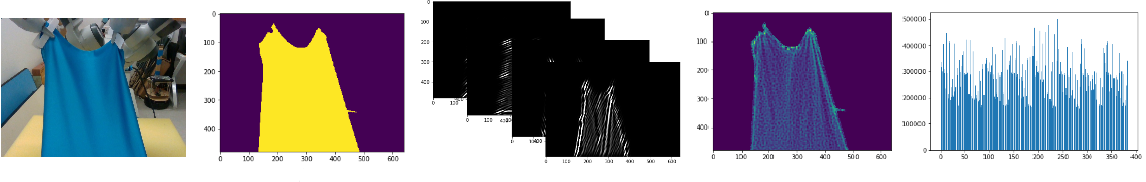}
\put(-470,0){(1)}
\put(-375,0){(2)}
\put(-265,0){(3)}
\put(-155,0){(4)}
\put(-50,0){(5)}
  \caption{
  Pipeline for HOW-feature computation:
  Given the input image (1), we use the following stages:
  (2) foreground segmentation using Gaussian mixture;
  (3) image filtering with multiple orientations and wavelengths of a Gabor Kernel;
  (4) discretization of the filtered images to form grids of a histogram;
  (5) stacking the feature matrix to a single column vector.}
  \label{fig:HOW:feature}
\end{figure*}

\begin{figure}[t]
  \centering
  \includegraphics[width=0.5\textwidth]		{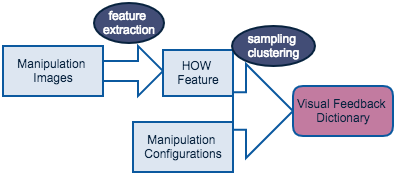}
  \caption{ 
  Computing the visual feedback dictionary: The input to this offline process is the recorded manipulation data with images and the corresponding configuration of the robot's end-effector. The output is a visual feedback dictionary, which links the velocity of the features and the controller.}
  \label{fig:HOW:offline}
\end{figure}

\begin{figure}[t]
  \centering
  \includegraphics[width=0.5\textwidth]{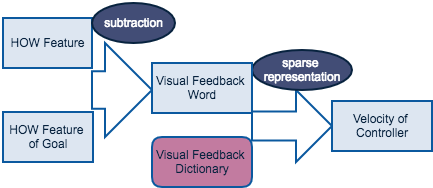}
  \caption{
  At runtime, the controller leverages the HOW features in two stages. First, we extract HOW features from the visual input and compute the visual feedback word by subtracting the extracted features from the features of the goal configuration. Then, we apply the sparse representation and compute the velocity of the controller for manipulation.}
  \label{fig:HOW:online}
\end{figure}

\subsection{Histogram of deformation model feature}
\label{sec:featuredesign2:feat}
Here we present our algorithm to compute the HOW-features from the camera stream. These are low-dimensional features of highly deformable material.   

The pipeline of our HOW-feature computation process is shown in Figure~\ref{fig:HOW:feature}, and it has three stages:
\subsubsection{Foreground segmentation}
To find the foreground partition of a deformable object, we apply the Gaussian mixture algorithm~\cite{lee:2005:effective} on the RGB data captured by the camera. The intermediate result of segmentation is shown in Figure~\ref{fig:HOW:feature}(2). 

\subsubsection{Deformation enhancement}
To model the high dimensional characteristics of the highly deformable material, we use deformation enhancement. This is based on the perceptual observation that most deformations can be modeled by shadows and shape variations. 
Therefore, we extract the features corresponding to shadow variations by applying a Gabor transform~\cite{lee:1996:image} to the RGB image. This results in the enhancement of the ridges, wrinkles, and edges (as shown in Figure ~\ref{fig:HOW:feature}). 
We convolve the $N$ deformation filters $\{d_i\}$  to the image $I$ and represent the result as $\{d_i(I)\}$.

In the spatial domain, a 2D Gabor filter is a Gaussian kernel function modulated by a sinusoidal plane wave~\cite{daugman:1985:uncertainty} and it has been used to detect wrinkles~\cite{yamazaki:2009:cloth}. The 2D Gabor filter can be represented as follows:
\begin{equation}
g(x,y;\lambda, \theta, \phi, \sigma, \gamma) = \exp(-\frac{ {x'}^2 + \gamma ^2 {y'}^2}{2 \sigma ^2}) \sin(2 \pi \frac{x'}{\lambda} + \phi),
\end{equation}
where $x' = x \cos\theta + y \sin\theta$, $y' = -x \sin\theta + y\cos\theta$, $\theta$ is the orientation of the normal to the parallel stripes of the Gabor filter, $\lambda$ is the wavelength of the sinusoidal factor, $\phi$ is the phase offset, $\sigma$ is the standard deviation of the Gaussian, and $\gamma$ is the spatial aspect ratio. When we apply the Gabor filter to our deformation model image, the choices for  wavelength ($\lambda$) and orientation ($\theta$) are the key parameters with respect to the wrinkles of deformable materials. As a result, the deformation model features consist of multiple Gabor filters ($d_{1, \cdots, n}(I)$) with different values of wavelengths ($\lambda$) and orientations ($\theta$).

\subsubsection{Grids of histogram}
A histogram-based feature is an approximation of the image which can reduce the data redundancy and extract a high-level representation that is robust to local variations in an image. Histogram-based features have been adopted to achieve a general framework for photometric visual servoing~\cite{bateux:2017:histograms}.
Although the distribution of the pixel value can be represented by a histogram, it is also significant to represent the position in the feature space of the deformation to achieve the manipulation task. Our approach is inspired by the study of grids in a Histogram of Oriented Gradient~\cite{dalal:2005:histograms}, which is computed on a dense grid of uniformly spatial cells. 

We compute the grids of histogram for the deformation model feature by dividing the image into small spatial regions and accumulating local histogram of different filters ($d_i$) of the region. For each grid, we compute the histogram in the region and represent it as a matrix. We vary the grid size and compute matrix features for each size. Finally, we represent the entries of a matrix as a column feature vector. 
The complete feature extraction process is described in Algorithm~\ref{alg:feat}.

\begin{algorithm}[!htp!]
  \caption{Computing HOW features}
  \label{alg:feat}
  \begin{algorithmic}[1]
   \renewcommand{\algorithmicrequire}{\textbf{Input:}}
   \renewcommand{\algorithmicensure}{\textbf{Output:}}
    \Require image $I$ of size $(w_I, h_I)$, deformation filtering or Gabor kernels $\{d_1 \cdots d_{N_d}\}$, grid size set$\{g_1,\cdots,g_{N_g}\}$.
    \Ensure feature vector $\mathbf s$
	\For {$i=1,\cdots,N_d$}
    \For {$j=1,\cdots,N_g$}
      \For {$(w,h)=(1,1),\cdots, (w_I,h_I)$}
          \LineComment{compute the indices using truncation}
          \State $(x,y) =(\textsc{trunc}(w / g_j), \textsc{trunc}(h / g_j))$ 
          \LineComment{add the filtered pixel value to the specific bin of the grid}
          \State $\mathbf s_{i,j,x,y} = \mathbf s_{i,j,x,y} + d_i(I(w,h))$ 
	  \EndFor
    \EndFor
    \EndFor
    \State \Return $\mathbf s$
  \end{algorithmic}
\end{algorithm}

The HOW-feature has several advantages. It captures the deformation structure, which is based on the characteristics of the local shape.
Moreover, it uses a local representation that is invariant to local geometric and photometric transformations. This is useful when the translations or rotations are much smaller than the local spatial or orientation grid size.

\subsection{Manipulation using the visual feedback dictionary}
\label{sec:HOW:learning}
Here we present our algorithm for computing the visual feedback dictionary. At runtime, this dictionary is used to compute the corresponding velocity ($\Delta \mathbf r$) of the controller based on the visual feedback ($\Delta s(I)$). 

\begin{figure}[b]
  \centering
  \includegraphics[width=0.5\textwidth]{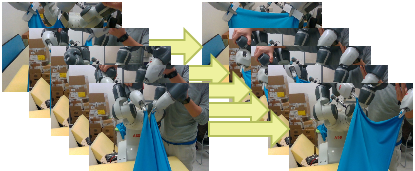}
  \caption{
  Visual feedback dictionary: The visual feedback word is defined by the difference between the visual features $\Delta \mathbf s = \mathbf s - \mathbf s^*$ and the controller positions $\Delta \mathbf r = \mathbf r - \mathbf r^*$. The visual feedback dictionary $\{\{\Delta \mathbf s^{(i)}\},\{\Delta \mathbf r^{(i)}\}\}$ consists of the visual feedback words computed. We show the initial and final states on the left and the right, respectively.
  }
  \label{fig:HOW:dictionary}
\end{figure}

\subsubsection{Building the visual feedback dictionary}
As shown in Figure~\ref{fig:HOW:offline}, the inputs to the offline training phase are a set of images and end-effector configurations ($\{I^{(t)}\}$, $\{\mathbf r^{(t)}\}$) and the output is the visual feedback dictionary ($\big\{\{\Delta \mathbf s^{(i)}\},\{ \Delta \mathbf r^{(i)}\}\big\}$). 

For the training process, the end-effector configurations, ($\{\mathbf r^{(t)}\}$), are either collected by human tele-operation or generated randomly. A single configuration ($\mathbf r^{(i)}$) of the robot is a column vector of length $N_{dof}$, the number of degrees-of-freedom to be controlled. $\mathbf r \in \mathbb{R}^{N_{dof}}$ and its value is represented in the configuration space.

In order to compute the mapping $H(\cdot)$ in Equation~\ref{eq:mapping} from the visual feedback to the velocity, we need to transform the configurations $\{\mathbf r^{(t)}\}$ and image stream $\{I^{(t)}\}$ into velocities $\{\Delta \mathbf r^{(t)}\}$ and the visual feature feedback $\{\Delta \mathbf s^{(t)}\}$, respectively. One solution is to select a fixed time step $\Delta t$ and to represent the velocity in both the feature and the configuration space as:
\begin{align}
\Delta \mathbf r^{(t)} &= \frac{\mathbf r^{(t+\frac{\Delta t}{2})} - \mathbf r^{(t-\frac{\Delta t}{2})}}{\Delta t / C_{fr}}, \\
\Delta \mathbf s(I^{(t)}) &= \frac{\mathbf s(I^{(t+\frac{\Delta t}{2})}) - \mathbf s(I^{(i-\frac{\Delta t}{2})})}{\Delta t / C_{fr}},
\end{align}
where $C_{fr}$ is the frame rate of the captured video.

However, sampling by a fixed time step ($\Delta t$) leads to a limited number of samples ($N_{I} - \Delta t$) and can result in over-fitting. To overcome this issue, we break the sequential order of the time index to generate more training data from ${I^{(t)}}$ and ${\mathbf r^{(t)}}$. 
In particular, we assume that the manipulation task can be observed as a Markov process~\cite{baum:1972:inequality} and that each step is independent from every other. In this case, the sampling rates are given as follows, (when the total sampling amount is $n$):
\begin{align}
\Delta \mathbf r^{(t)} &= \frac{\mathbf r^{(l_{t})}- \mathbf r^{(l_{t+n})}} { (l_{t} - l_{t+n}) / C_{fr}}, \\
\Delta \mathbf s(I^{(t)}) &= \frac{\mathbf s(I^{(l_{t})}) - \mathbf s(I^{(l_{t+n})})} { (l_{t} - l_{t+n}) / C_{fr}},
\end{align} 
where $l_{1,\cdots, 2n}$ is a set of randomly generated indices and $l_t \in [1, \cdots, N_I]$.
To build a more concise dictionary, we also apply K-Means Clustering~\cite{hartigan:1979:algorithm} on the feature space, which enhances the performance and prevents the over-fitting problem.   

In practice, the visual feedback dictionary can be regarded as an approximation of the interaction function $H$ in Equation~\ref{eq:mapping}). The overall algorithm to compute the dictionary is given in Algorithm~\ref{alg:training}.

\subsubsection{Sparse representation}
At runtime, we use sparse linear representation~\cite{donoho:2003:optimally} to compute the velocity of the controller from the visual feedback dictionary. 
These representations tend to assign zero weights to most irrelevant or redundant features and are used to find a small subset of the most predictive features in the high dimensional feature space.
Given a noisy observation of a feature at runtime ($\mathbf s$) and the visual feedback dictionary constructed by features $\{\Delta \mathbf s^{(i)}\}$ with labels $\{\Delta \mathbf r^{(i)}\}$, we represent $\Delta \mathbf s$ by $\hat{\Delta \mathbf s}$, which is a sparse linear combination of $\{\Delta \mathbf s^{(i)}\}$, where $\beta$ is the sparsity-inducing $L_1$ term. To deal with noisy data, we use the $L_2$ norm on the data-fitting term and formulate the resulting sparse representation as:
\begin{equation}
\hat{\beta} = \argmin_\beta\big( \|\min(\Delta \mathbf s - \sum_i \beta_i \Delta \mathbf s^{(i)} )\|_2^2+\alpha \|\beta\|_1 \big), \label{alpha1}
\end{equation}
where $\alpha$ is a slack variable that is used to balance the trade-off between fitting the data perfectly and using a sparse solution.
The sparse coefficient $\beta^*$ is computed using a minimization formulation:
\begin{equation}
   \beta^* =  \argmin_\beta \big(\sum_{i} {\rho(\Delta \mathbf s_i^* - \sum_{j} {\beta_j \Delta \mathbf s^{(j)}_i })} \label{alpha2}
        + \alpha \sum_{j} {\|\beta_j\|_1}\big).
\end{equation}
After $\hat{\beta}$ is computed, the observation $\Delta \hat{\mathbf s}$ and the probable label $\Delta \hat{\mathbf r}$ can be reconstructed by the visual feedback dictionary:
\begin{equation}
\Delta \hat{\mathbf s} = \sum_i \hat{\beta_i} \Delta \mathbf s^{(i)} \quad \quad 
\Delta \hat{\mathbf r} = \sum_i \hat{\beta_i} \Delta \mathbf r^{(i)}
\end{equation}
The corresponding ${\Delta \mathbf r}^*$ of the $i-$th DOF in the configuration is given as: 
\begin{equation}
   \Delta \mathbf r_i^* = \sum_{j} \beta^*_j \mathbf \Delta r_i^{(j)}, 
\end{equation}
where $\Delta \mathbf s^{(j)}_i$ denotes the $i$-th datum of the $j$-th feature, $\Delta \mathbf s^*_i$ denotes the value of the response, and the norm-1 regularizer $\sum_{j} {\|\beta_j\|_1}$ typically results in a sparse solution in the feature space. 

\begin{algorithm}[t]
  \caption{Building the visual feedback dictionary}
  \label{alg:training}
  \begin{algorithmic}[1]
   \renewcommand{\algorithmicrequire}{\textbf{Input:}}
   \renewcommand{\algorithmicensure}{\textbf{Output:}}
    \Require image stream $\{I^{(t)}\}$ and positions of end-effectors $\{\mathbf r^{(t)}\}$ with sampling amount $n$, dictionary size $N_{dic}$
    \Ensure Visual feedback dictionary $\{\{\Delta \mathbf s_d^{(i)}\},\{ \Delta \mathbf r_d^{(i)}\}\}$
    \State $\{\Delta \mathbf s_d\} = \{\}$, $\{\Delta \mathbf r_d\} = \{\}$
    \LineComment{generate random indices for sampling}
    \State $l = N_I \textsc{rand}(2n)$ 
    \For {$i = 1, \cdots, n$}
    	\State $\Delta \mathbf s^{(i)}$ = $s(I^{(l(i))}) - s(I^{(l(i+n))})$ \Comment{sampling}
    	\State $\Delta \mathbf r^{(i)}$ = $\mathbf r^{(l(i))} - \mathbf r^{(l(i+n))}$ \Comment{sampling}
    \EndFor
    \LineComment{compute the centers of the feature set for clustering}
    \State  $centers = \textsc{K-MEANS}(\{\Delta \mathbf s^{(i)}\}, N_{dic})$ 
    \For {$i = 1, \cdots, N_{dic}$}
    	\State $j = \argmin_i \|centers[i] - \mathbf s^{(i)})\|$
        \State $\{\Delta \mathbf s_d\} = \{\Delta \mathbf s_d, \Delta \mathbf s^{(j)}\}$ $\{\Delta \mathbf r_d\} = \{\Delta \mathbf r_d, \Delta \mathbf r^{(j)}\}$
    \EndFor
   \State \Return $\{\Delta \mathbf s_d\},\{\Delta \mathbf r_d\}$ 
  \end{algorithmic}
\end{algorithm}

\subsubsection{Goal configuration and mapping}
\label{sec:goal}
We compute the goal configuration and the corresponding HOW-features based on the underlying manipulation task at runtime. Based on the configuration, we compute the velocity of the end-effector. The different ways to compute the goal configuration are:
\begin{itemize}
\item To manipulate deformable objects to a single state $c^*$, the goal configuration can be represented simply by the visual feature of the desired state $\mathbf s^*=\mathbf s(I^*)$.
\item To manipulate deformable objects to a hidden state $c^*$, which can be represented by a set of states of the object $c^* = \{c_1, \cdots, c_n\}$
as a set of visual features $\{\mathbf s(I_1),\cdots,\mathbf s(I_n)\}$. We modify the formulation in Equation~\ref{eq:mapping} to compute $\delta \mathbf r$ as:
\begin{equation}
\label{eq:f_h1}
\delta \mathbf r = \min_i H\Big(\eta \big(\mathbf s(I) - \mathbf s(I_j) \big) \Big)
\end{equation}
\item For a complex task, which can be represented using a sequential set of states $\{c_1, \cdots, c_n\}$, we estimate the sequential cost of each state as $\text{cost}(c_i)$. We use a modification that tends to compute the state with the lowest sequential cost:
\begin{equation}
\label{eq:f_h2}
i^* = \argmin_i \bigg (\text{cost}(c_i) + H\Big(\eta \big(\mathbf s(I)- \mathbf s(I_i) \big) \Big) \bigg).
\end{equation}
After $i^*$ is computed, the velocity for state $c_i$ is determined by $\mathbf s(I_{i^*})$, and $c_i$ is removed from the set of goals for subsequent computations.
\end{itemize}

\section{Controller design: nonlinear Gaussian Process Regression}
\label{sec:controllerdesign}
\subsection{Interaction function learning}
\label{sec:ldf}
Unlike many previous methods that assume the interaction function $H(\cdot)$ to be a linear function, here we consider $H(\cdot)$ as a general and highly nonlinear function determining how the movement of the manipulated points is converted into the feature space. Learning the function $H$ requires a flexible and non-parametric method. Our solution is to use Gaussian Process Regression (GPR) to fit the interaction function in an online manner.  

GPR is a nonparametric regression technique that defines a distribution over functions and in which the inference takes place directly in the functional space given the covariance and mean of the functional distribution. For our manipulation problem, we formulate the interaction function $H$ as a Gaussian process:
\begin{equation}
	H \sim GP \left( m(\delta \mathbf s),\,\, k(\delta \mathbf s, \delta \mathbf s') \right)
\end{equation}
where $\delta \mathbf s$ still denotes the velocity in the feature space.  For the covariance or kernel function $k(\delta \mathbf s,\delta \mathbf s')$, we use the Radius Basis Function (RBF) kernel: $k(\delta \mathbf s,\delta \mathbf s') = \exp(-\frac{\|\delta \mathbf s- \delta \mathbf s'\|^2}{2\sigma_{RBF}^2})$, where the parameter $\sigma_{RBF}$ sets the spread of the kernel. For the mean function $m(\delta \mathbf s)$, we use the linear mean function $m(\delta \mathbf s) = \mathbf W \delta \mathbf s$, where $\mathbf W$ is the linear regression weight matrix. We choose to use a linear mean function rather than the common zero mean function, because previous work~\cite{Alarcon:2016:ADM} has shown that the adaptive Jacobian method, which can be considered as a special version of our method using the linear mean function, is able to capture a large part of the interaction function $H$. As a result, a linear mean function can result in faster convergence of our online learning process and provide a relatively accurate prediction in the unexplored region in the feature space. The matrix $W$ is learned online by minimizing a squared error $Q = \frac{1}{2}\|\delta \mathbf r-\mathbf W\delta \mathbf s\|^2$ with respect to the weight matrix $\mathbf W$. 

Given a set of training data in terms of pairs of feature space velocities and manipulated point velocities $\{(\delta \mathbf s_t, \delta \mathbf r_t)\}_{t=1}^N$ during the previous manipulation process, the standard GPR
computes the distribution of the interaction function as a Gaussian process $H(\delta \mathbf s) \sim \mathcal N(\mu(\delta \mathbf s), \sigma^2(\delta \mathbf s))$, where GP's mean function is 
\begin{align}
\label{eq:mu}
\mu(\delta \mathbf s) &= m(\delta \mathbf s) + \mathbf k^T(\delta \mathbf S, \delta \mathbf s)\cdot [\mathbf K(\delta \mathbf S, \delta \mathbf S) + \sigma_n^2 \mathbf I]^{-1}\notag \\
&\quad\cdot (\delta \mathbf R -m(\delta \mathbf S)) 
\end{align}
and GP's covariance function is 
\begin{align}
\label{eq:sigma}
\sigma^2(\delta \mathbf s) &= k(\delta \mathbf s,\delta \mathbf s) - \mathbf k^T(\delta \mathbf S,\delta \mathbf s)\cdot [\mathbf K(\delta \mathbf S,\delta \mathbf S) + \sigma_n^2 \mathbf I]^{-1}  \notag\\
&\quad \cdot \mathbf k(\delta \mathbf S, \delta \mathbf s). 
\end{align}
Here $\delta \mathbf S$ and $\delta \mathbf R$ are matrices corresponding to the stack of $\{\delta \mathbf s_t\}_{t=1}^N$ and $\{\delta \mathbf r_t\}_{t=1}^N$ in the training data, respectively. $\mathbf K$ and $\mathbf k$ are matrices and vectors computed using a given covariance function $k(\cdot, \cdot)$. The matrix $\mathbf A = \mathbf K + \sigma_n^2 \mathbf I$ is called the Gram matrix, and the parameter $\sigma_n$ estimates the uncertainty or noise level of the training data. 

\subsection{Real-time online GPR}
\label{sec:FO-GPR}

In the deformation object manipulation process, the data $(\delta \mathbf s_t, \delta \mathbf r_t)$ is generated sequentially. Thus, at each time step $t$, we need to update the GP interaction function $H_t(\delta \mathbf s) \sim \mathcal N(\mu_t(\delta \mathbf s), \sigma^2(\delta \mathbf s))$ in an interactive manner, with
\begin{align}
\label{eq:mut}
\mu_t(\delta \mathbf s) &= m(\delta \mathbf s) + \mathbf k^T(\delta \mathbf S_t, \delta \mathbf s) \cdot [\mathbf K(\delta \mathbf S_t, \delta \mathbf S_t) + \sigma_n^2 \mathbf I]^{-1} \notag \\
&\quad \cdot (\delta \mathbf R_t -m(\delta \mathbf S_t))
\end{align}
and
\begin{align}
\label{eq:sigmat}
\sigma^2_t(\delta \mathbf s) &= k(\delta \mathbf s,\delta \mathbf s) - \mathbf k^T(\delta \mathbf S_t,\delta \mathbf s)\cdot [\mathbf K(\delta \mathbf S_t,\delta \mathbf S_t) + \sigma_n^2 \mathbf I]^{-1} \notag \\
&\quad \cdot \mathbf k(\delta \mathbf S_t, \delta \mathbf s).
\end{align}

In the online GPR, we need to perform the inversion of the Gram matrix $\mathbf A_t = \mathbf K(\delta \mathbf S_t, \delta \mathbf S_t) + \sigma_n^2 \mathbf I$ repeatedly with a time complexity $\mathcal O(N^3)$, where $N$ is the size of the current training set involved in the regression. Such cubic complexity makes the training process slow for long manipulation sequences where the training data size $N$ increases quickly. In addition, the growing up of the GP model will reduce the newest data's impact on the regression result and make the GP fail to capture the change of the objects's deformation parameters during the manipulation. This is critical for deformable object manipulation because the interaction function $H$ is derived from the local force equilibrium and thus is only accurate in a small region. 

Motivated by previous work about efficient offline GPR~\cite{Snelson:2006:SGP, Rasmussen:2002:MGP, Snelson:2007:GPA, Duy:2009:LGP}, we here present a novel online GPR method called the \emph{Fast Online GPR} (FO-GPR) to reduce the high computational cost and to adapt to the changing deformation properties while updating the deformation model during the manipulation process. The main idea of FO-GPR includes two parts: 1) maintaining the inversion of the Gram matrix $\mathcal A_t$ incrementally rather using direct matrix inversion; 2) restricting the size of $\mathbf A_t$ to be smaller than a given size $M$ and, if $\mathbf A_t$'s size exceeds that limit, using a selective ``forgetting'' method to replace stale or uninformative data with a fresh new data point. 

\subsubsection{Incremental update of Gram matrix $\mathbf A_t$} Suppose at time $t$, the size of $\mathbf A_t$ is still smaller than the limit $M$. In this case, $A_t$ and $A_{t-1}$ are related by  
\begin{equation}
	\mathbf A_t = 
  \begin{bmatrix}
	\mathbf A_{t-1} & \mathbf b \\ \mathbf b^T & c
  \end{bmatrix},
\end{equation}
where $\mathbf b = \mathbf k(\delta \mathbf S_{t-1},\delta \mathbf s_t)$ and $c = k(\delta \mathbf s_t,\delta \mathbf s_t) + \sigma^2_n$. According to the Helmert–Wolf blocking inverse property, we can compute the inverse of $\mathbf A_t$ based on the inverse of  $\mathbf A_{t-1}$:
\begin{align}
\label{eq:incremA}
	\mathbf A_t^{-1} & = 
    \begin{bmatrix}
    \left(\mathbf A_{t-1}-\frac{1}{c}\mathbf {bb}^T \right)^{-1} & -\frac{1}{r}\mathbf A_{t-1}^{-1} \mathbf b \\ -\frac{1}{r}\mathbf b^T\mathbf A_{t-1}^{-1} & \frac{1}{r}
    \end{bmatrix} \notag \\
    & = 
    \begin{bmatrix}
    \mathbf A_{t-1}^{-1} + \frac{1}{r} \mathbf A_{t-1}^{-1} \mathbf {bb}^T\mathbf A_{t-1}^{-1} & -\frac{1}{r}\mathbf A_{t-1}^{-1} \mathbf b \\ -\frac{1}{r}\mathbf A_{t-1}^T\mathbf A_{t-1}^{-1} & \frac{1}{r}
    \end{bmatrix},
\end{align}
where $r = c-\mathbf b^T\mathbf A_{t-1}^{-1}\mathbf b$. In this way, we achieve the incremental update of the inverse Gram matrix from $\mathbf A_{t-1}^{-1}$ to $\mathbf A_t^{-1}$, and its computational cost is $\mathcal O(N^2)$ rather than $\mathcal O(N^3)$ of direct matrix inversion. This acceleration enables fast GP model updates during the manipulation process.

\subsubsection{Selective forgetting in online GPR} When the size of $\mathbf A_{t-1}$ reaches the limit $M$, we use a ``forgetting'' strategy to replace the most uninformative data with the fresh data points while keeping the size of $\mathbf A_t$ to be $M$. In particular, we choose to forget the $i_*$ data point that is the most similar to other data points in terms of the covariance, i.e.,  
\begin{equation}
\label{eq:opti}
 i_* = \argmax_{i} \sum\nolimits_j\mathbf A[i, j],
\end{equation}
where $\mathbf A[i,j]$ denotes the covariance value stored in the $i$-th row and $j$-th column in $\mathbf A$, i.e., $k(\delta \mathbf s_i, \delta \mathbf s_j) + \sigma_n^2$.

Given the new data $(\delta \mathbf s_t, \delta \mathbf r_t)$, we need to update $\delta \mathbf S_t$, $\delta \mathbf R_t$, and $\mathbf A_t^{-1} = [\mathbf K(\delta \mathbf S_t, \delta \mathbf S_t) + \sigma_n^2 \mathbf I]^{-1}$ in Equations~\ref{eq:mut} and~\ref{eq:sigmat} by swapping data terms related to $\delta \mathbf s_t$ and $\delta \mathbf s_{i_*}$, to update the interaction function $H_t$. 

The incremental update for $\delta \mathbf S_t$ and $\delta \mathbf R^t$ is trivial: $\delta \mathbf S_t$ is identical to $\delta \mathbf S_{t-1}$ except $\delta \mathbf S_t [i_*]$ is $\delta \mathbf s_t$ rather than $\delta \mathbf s_{i_*}$; $\delta \mathbf R_t$ is identical to $\delta \mathbf R_{t-1}$ except $\delta \mathbf R_t [i_*]$ is $\delta \mathbf r_t$ rather than $\delta \mathbf r_{i_*}$.

We then discuss how to update $\mathbf A_t$ from $\mathbf A_{t-1}$. Since $\mathbf A_t - \mathbf A_{t-1}$ is only non-zero at the $i_*$-th column or the $i_*$-th row:
\begin{equation}
(\mathbf A_t - \mathbf A_{t-1})[i,j]=\begin{cases}
0, & i,j \neq i_* \\
k(\delta \mathbf s_i, \delta \mathbf s) - k(\delta \mathbf s_i, \delta \mathbf s_{i_*}), & j = i_* \\
k(\delta \mathbf s_j, \delta \mathbf s) - k(\delta \mathbf s_j, \delta \mathbf s_{i_*}), & i = i_*.
\end{cases} \notag
\end{equation}
This matrix can be written as the multiplication of two matrices $\mathbf U$ and $\mathbf V$, i.e., $\mathbf A_t - \mathbf A_{t-1} = \mathbf U \mathbf V^T$, where
\begin{equation}
\mathbf U = 
\begin{bmatrix}
\mathbf e_{i_*} & (\mathbf I -\frac{1}{2} \mathbf e_{i_*} \mathbf e_{i_*}^T) (\mathbf k_t - \mathbf k_{t-1})
\end{bmatrix} \notag
\end{equation}
and 
\begin{equation}
\mathbf V = 
\begin{bmatrix}
(\mathbf I -\frac{1}{2} \mathbf e_{i_*} \mathbf e_{i_*}^T) (\mathbf k_t - \mathbf k_{t-1}) & \mathbf e_{i_*}
\end{bmatrix}. \notag
\end{equation}
Here $\mathbf e_{i_*}$ is a vector that is all zero except when it is one at the $i_*$-th item, $\mathbf k_t$ is the vector $\mathbf k(\delta \mathbf S_{t}, \delta \mathbf s_t)$ and $\mathbf k_{t-1}$ is the vector $\mathbf k(\delta \mathbf S_{t-1}, \delta \mathbf s_{i_*})$. Both $\mathbf U$ and $\mathbf V$ are size $M \times 2$ matrices. 

Then, using the Sherman-Morrison formula, we get
\begin{align}
\label{eq:incremA2}
\mathbf A_t^{-1} &= (\mathbf A_{t-1} + \mathbf U \mathbf V^T)^{-1}\\
&= \mathbf A_{t-1}^{-1} - \mathbf A_{t-1}^{-1}\mathbf U (\mathbf I + \mathbf V^T \mathbf A_{t-1}^{-1} \mathbf U)^{-1} \mathbf V^T \mathbf A_{t-1}^{-1},  \notag 
\end{align}
which provides the incremental update scheme for the Gram matrix $\mathbf A_t$. Since $\mathbf {I+V}^T\mathbf A_{t-1}^{-1} \mathbf U$ is a $2\times 2$ matrix, its inversion can be computed in $\mathcal O(1)$ time. Therefore, the incremental update computation is dominated by the matrix-vector multiplication and thus the time complexity is $\mathcal O(M^2)$ rather than $\mathcal O(M^3)$. 

A complete description of FO-GPR is shown in Algorithm~\ref{algo:FOGPR}.

\begin{algorithm}
 \caption{FO-GPR}
 \label{algo:FOGPR}
 \begin{algorithmic}[1]
 \renewcommand{\algorithmicrequire}{\textbf{Input:}}
 \renewcommand{\algorithmicensure}{\textbf{Output:}}
 \Require $\delta \mathbf S_{t-1}$, $\delta \mathbf R_{t-1}$, $\mathbf A_{t-1}$, $\delta \mathbf s_t$, $\delta \mathbf r_t$
 \Ensure $\delta \mathbf S_t$, $\delta \mathbf R_t$, $\mathbf A_t^{-1}$ 
 \If{$dim(\mathbf A_{t-1}) < M$}
 \State $\delta \mathbf S_t = [\delta \mathbf S_{t-1}, \delta \mathbf s_t]$ 
 \State $\delta \mathbf R_t = [\delta \mathbf R_{t-1}, \delta \mathbf r_t]$ 
 \State $\mathbf A_t^{-1}$ computed using Equation~\ref{eq:incremA}
 \Else
 \State $i_*$ computed using Equation~\ref{eq:opti} 
 \State $\delta \mathbf S_t = \delta \mathbf S_{t-1}$, $\delta \mathbf S_t[i_*] = \delta \mathbf s_t$ 
  \State $\delta \mathbf R_t = \delta \mathbf R_{t-1}^m$, $\delta \mathbf R_t[i_*] = \delta \mathbf r_t$ 
  \State $\mathbf A_t^{-1}$ computed using Equation~\ref{eq:incremA2}
 \EndIf
 \end{algorithmic}
\end{algorithm}

\subsection{Exploitation and exploration}
\label{sec:exploration}
Given the interaction function $H_t$ learned by FO-GPR, the controller system predicts the required velocity to be executed by the end-effectors based on the error between the current state and the goal state in the feature space:
\begin{align}
\label{eq:predic}
\delta \mathbf r = H_t(\eta \cdot (\mathbf s^* - \mathbf s)).
\end{align}

However, when there is no sufficient data, GPR cannot output control policy with high confidence, which typically happens in the early steps of the manipulation or when the robot manipulates the object into a new unexplored configuration. Fortunately, the GPR framework provides a natural way to trade-off exploitation and exploration by sampling the control velocity from distribution of $H_t$:
\begin{equation}
	\delta \mathbf r \sim \mathcal N(\mu_t, \sigma_t^2)
\end{equation}
If $\mathbf s_t$ is now in an unexplored region with a large $\sigma_t^2$, the controller will perform exploration around $\mu_t$; if $\mathbf s_t$ is in a well-explored region with a small $\sigma_t^2$, the controller will output a velocity close to $\mu_t$. 

A complete description of the controller based on FO-GPR is shown in Figure~\ref{fig:flow-chart}.
 
\begin{figure*}[t] 
\centering
\includegraphics[width=0.8\linewidth]{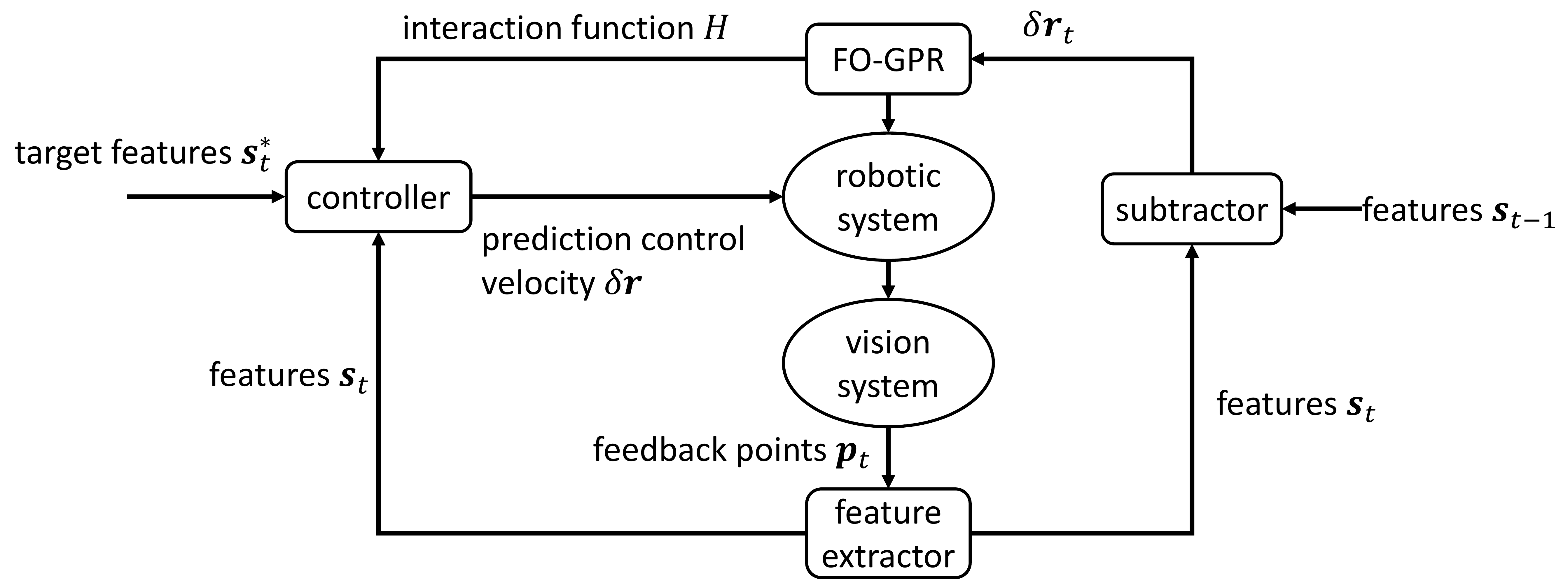}
\caption{An overview of our nonlinear DOM controller based on Gaussian Process Regression. }
\label{fig:flow-chart}
\end{figure*}

\subsection{Convergence and stability analysis}
\label{sec:convergence}
We can prove that, given more and more data, our GPR-based online learning converges to the true underlying deformation distribution and that the resulting controller is stable. 

The GPR prediction in Equation~\ref{eq:mu} and~\ref{eq:predic} includes two terms: The first term $m(\delta \mathbf s)$ actually corresponds to the adaptive Jacobian in~\cite{Alarcon:2016:ADM} which has proven to be bounded and is able to asymptotically minimize the error between the current and target states. The second term corresponds to a function $H$, which minimizes the functional
\begin{equation}
\label{lossfunc}
	J[H] = \frac{1}{2}\|H\|^2_{RHKS} + \frac{1}{2\sigma_n^2}\sum_{t=1}^N(\delta \mathbf R_t - H(\delta \mathbf s_t))^2,
\end{equation}
where $\|H\|_{RHKS}$ is the RKHS (reproducing kernel Hilbert space) norm w.r.t. the kernel $k$. According to~\cite{rasmussen2006gaussian}, we can prove that this second term converges:
\begin{prop}
The prediction function $H$ converges to a true underlying interaction function $\eta$.
\end{prop}
\begin{proof}
Let $\mu(\delta \mathbf S,\delta \mathbf R)$ be the probability measure from which the data pairs$(\delta \mathbf s_t,\delta \mathbf S_t)$ are generated. We have
\begin{align}
	\mathbb E[\sum_{t=1}^N(\delta \mathbf R_t - H(\delta \mathbf s_t))^2] &= N\int(\delta \mathbf R - H(\delta \mathbf S))^2\notag d\mu.
\end{align}
Let $\eta(\delta \mathbf S) = \mathbb E[\delta \mathbf R|\delta \mathbf S]$ and $\lambda_i$ be the $i$-th eigenfunction of the kernel function $k$,the $J[H]$ functional then becomes
\begin{equation}
J[H] = \frac{N}{2\sigma_n^2}\sum_{t=1}^\infty(\eta_t-H_t)^2+\frac{1}{2}\sum_{t=1}^\infty\frac{H_t^2}{\lambda_t}
\end{equation}
and its solution can then be computed as
\begin{equation}
\label{eq:gp_solu}
H_t = \frac{\lambda_t}{\lambda_t+\sigma_n^2/N}\eta_t.
\end{equation}
When $n \rightarrow \infty$, $\sigma_n^2/N \rightarrow 0$ and thus $H$ converges to $\eta$. 
\end{proof}

After showing the convergence of Equations~\ref{eq:mu} and~\ref{eq:predic}, we can perform local linearization for the nonlinear GPR prediction function and then design a Lyapunov function similar to~\cite{Alarcon:2016:ADM} to show the stability of the GPR-based visual-servoing controller.

\section{Controller design: random forest controller}
\label{sec:controllerdesign2}
The GPR controller proposed above has two main limitations. First, the controller must be in the form of a mathematical function, making it unable to model the more complicated controller required for dealing with large deformations. Second, the feature extraction step and control design step are independent.

Here, we present a another controller for DOM tasks. This controller uses random forest~\cite{quinlan:1986:induction} to model the mapping between the visual features of the object and an optimal control action of the manipulator. The topological structure of this random-forest-based controller is determined automatically based on the training data, which consists of visual features and control actions. 
This enables us to integrate the overall process of training data classification and controller optimization into an imitation learning (IL) algorithm. 
Our approach enables joint feature extraction and controller optimization for DOM tasks.

The pipeline of our approach is illustrated in Figure~\ref{fig:Forest:framework} and it consists of two components. For preprocessing (the blue block), a dataset of object observations is labeled and features are extracted for each observation. This dataset is used to train our controller. During runtime (the red block), our algorithm takes the current $\mathcal{O}(c)$ as input and generates the optimal action $\mathbf r^* = \pi(\mathcal{O}(c)|\beta)$, where $\pi$ is the control policy and $\beta$ is the controller's learnable parameter.

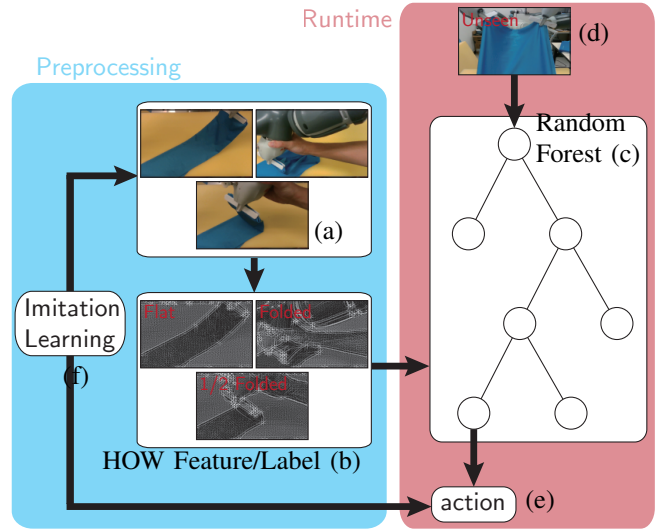
\begin{figure}[t]
  \centering
  \includegraphics[width=\linewidth]{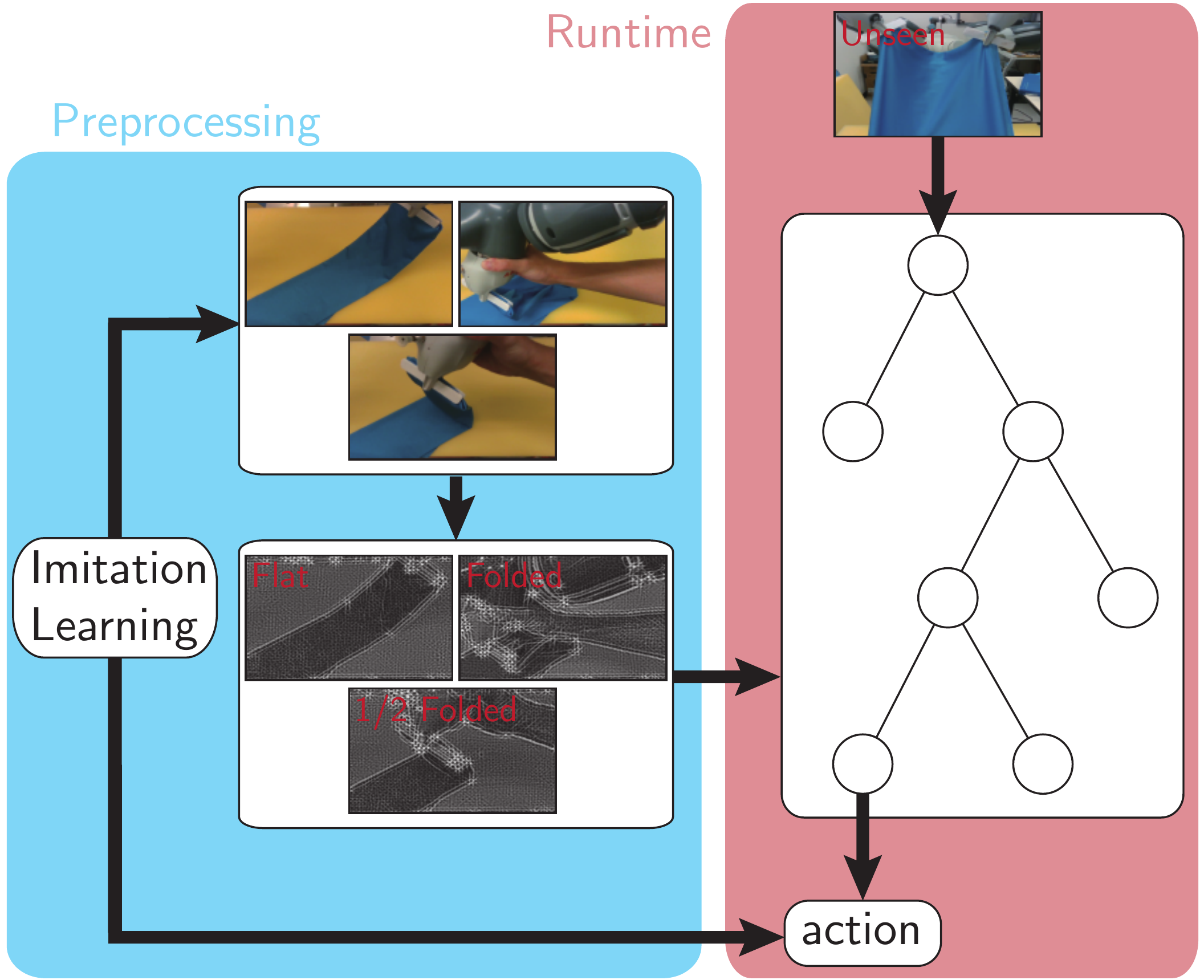}
  \put(-130,110){(a)}
  \put(-210,23){HOW Feature/Label (b)}
  \put(-46,145){\parbox{.7in}{Random Forest (c)}}
  \put(-30,185){(d)}
  \put(-50,8){(e)}
  \put(-225,55){(f)}
  \caption{
  The pipeline of learning a random forest-based DOM controller that maps the visual feature (RGB-D image) to the control action. Given a sampled dataset (a), we first label each data point (shown as red text in (b)) to get a labeled dataset, (b). We then construct a random forest to classify the images, (c).
After training, the random forest is used as a controller. Given an unseen visual observation (d), the observation is brought through the random forest to a set of leaf-nodes. The optimal control actions are defined on these leaf-nodes, (e).  
The entire process of labeling, classification, and controller optimization can be integrated into an imitation learning algorithm, (f).}
  \label{fig:Forest:framework}
\end{figure}

\subsection{Random forest-based controller formulation}
Our key contribution is a novel parametrization of $\pi$ with the use of a random forest~\cite{quinlan:1986:induction}. 
A random forest is an ensemble of $K$ decision trees, where the $k$-th tree classifies $\mathcal{O}(c)$ by bringing it to a leaf-node $l_k(\mathcal{O}(c))$, where $1\leq l_k(\mathcal{O}(c))\leq L_k$ and $L_k$ is the number of leaf-nodes in the $k$-th decision tree. 
The random forest makes its decision by bringing $\mathcal O(c)$ through every tree and taking the average.
In order to use an already constructed random forest as a controller, we define an optimal control action
 $r_{l,k}^*$ so that the final action is determined by averaging:
\begin{align}
\label{eq:ctrl_rf}
\mathbf r^* = \pi(\mathcal{O}(c)|\alpha,\beta) = \frac{1}{K} \mathbf r_{l_k(\mathcal{O}(c)),k}^*,
\end{align}
where we introduce an additional parameter $\alpha$ that encodes the random forest's topology. 
As a result, the number of optimizable controller parameters is $|\theta|=|\mathbf r^*|\sum_{k=1}^K L_k$. 
Such an extension to a random forest has two benefits.
First, even if two decision trees give the same classification for an observation $\mathcal{O}(c)$, i.e., the class label of $l_a(\mathcal{O}(c))$ equals that of $l_b(\mathcal{O}(c))$, we still assign separate optimal control actions, $\mathbf r_{l_a,a}$ and $\mathbf r_{l_b,b}$, for optimization. This makes the controller more robust to bad predictions. 
Second, the number of optimized parameters is related to the random forest's topology. This reveals strong connection between feature extraction and controller parametrization. 
By automatically segmenting the state space into pieces (leaf-nodes) where different control actions need to be taken, fewer design decisions on the controller's side are exposed to the end user. This makes our method less sensitive to parameters.

\subsection{Controller optimization problem}
Our controller training problem can take different forms depending on the available information about $c^*$. If $\mathcal{O}(c^*)$ is known, then we can define a reward function: $R(c)=\mathbf{dist}(\mathcal{O}(c),\mathcal{O}(c^*))$, where $\mathbf{dist}$ can be any distance measure between RGB-D images. In this setting, we want to solve the following reinforcement learning (RL) problem:
\begin{align}
\argmax_{\alpha,\beta} \mathbf{E}_{\tau\sim\pi}{\left[\sum_i^\infty\gamma^iR(c_i)\right]},
\end{align}
where $\tau=(c_1,c_2,\cdots,c_\infty)$ is a trajectory sampled according to $\pi$ and $\gamma$ is the discount factor. Another widely used setting assumes that $\mathcal{O}(c^*)$ is unknown, but that an expert is available to provide optimal control action $\pi^*(\mathcal{O}(c))$. In this case, we want to solve the following IL problem:
\begin{align}
\argmax_{\alpha,\beta} \mathbf{E}_{\tau\sim\pi}{\left[\sum_i^\infty\gamma^i\mathbf{dist}(\pi^*(\mathcal{O}(c_i)),\pi(\mathcal{O}(c_i)))\right]}.
\end{align}
Our method is designed for this IL problem and we assume that $\mathcal{O}(c^*)$ is known even in an IL setting. This is because $\mathcal{O}(c^*)$ is needed to construct the random forest. 
This assumption always holds when we are training in a simulated environment. 
We describe the method to perform these optimizations in Section~\ref{sec:technique}.

\subsection{Learning random forest-based controller}\label{sec:technique}
To find the controller parameters, we use an IL algorithm~\cite{Ross:2011:RIL}, which can be decomposed into two substeps: dataset sampling and controller optimization. The first step samples a dataset $\mathcal{D}=\{\langle \mathcal{O}(c),\mathbf r^*\rangle \}$, where each sample is a combination of cloth observation and optimal action. Our goal is to optimize the random forest-based controller with respect to $\alpha,\beta$, given $\mathcal{D}$. 

\subsubsection{Random forest construction}
\label{sec:app1}

We first solve for $\alpha$, the topology of the random forest. 
Unfortunately, a random forest construction requires a labeled dataset that classifies the samples into discrete sets of classes. 
To generate these labels, we run a mean-shift clustering computation on the optimal actions, $\mathbf r^*$. In addition, we reduce the dimension of the RGB-D image by extracting the HOW-feature (Section~\ref{sec:featuredesign2}).

After feature mapping, we get a modified dataset $\bar{\mathcal{D}}=\{\langle \mathcal{F}(\mathcal{O}(c)),\mathcal{L}(\mathbf r^*) \rangle\}$, where $\mathcal{F}$ is the feature extractor such that the extracted feature $\mathbf s = \mathcal{F}(\mathcal{O}(c))$ and $\mathcal{L}$ is the mean-shift labeling.

To construct the random forest, we use a strategy that is similar to~\cite{Shotton:2011:RHP}. We construct $K$ binary decision trees in a top-down manner, each using a random subset of $\bar{\mathcal{D}}$. Specifically, for each node of a tree, a set of random partitions is computed and the one with the maximal Shannon information gain~\cite{quinlan:1986:induction} is adopted. Each tree is grown until a maximum depth is reached or the best Shannon information gain is lower than a threshold.

\subsubsection{Controller optimization}
After constructing the random forest, we define optimal control actions on each leaf-node of each decision tree. The optimal control action can be derived according to Equation~\ref{eq:ctrl_rf}. 
We can then optimize for $\beta$, i.e., all the $\mathbf r_{l,k}^*$, by solving the following optimization problem:
\begin{align}
\argmin_{\mathbf r_{l,k}^*} \sum_{\langle \mathcal{O}(c),\mathbf r^* \rangle}
\|\mathbf r^*-\frac{1}{K} \mathbf r_{l_k(\mathcal{O}(c)),k}^*\|^2.
\end{align}
This is a quadratic energy function that can be solved in closed form. However, a drawback of this energy is that every decision tree contributes equally to the final optimal action, making it less robust to outliers. To resolve this problem, we introduce a more robust model by exploiting sparsity. We assume that each leaf-node has a confidence value denoted by $\gamma_{l,k}\in[0.01,1]$. The final optimal action is found by weighted averaging:
\begin{align}
\mathbf r^* = \pi(\mathcal{O}(c)|\alpha,\beta) = \frac{\sum\gamma_{l_k(\mathcal{O}(c)),k} \mathbf r_{l_k(\mathcal{O}(c)),k}^*}{\sum\gamma_{l_k(\mathcal{O}(c)),k}},
\end{align}
where a lower bound $(0.01)$ on $\gamma$ is used to avoid division-by-zero. 
Our final controller optimization takes the following form:
\begin{align}
\label{eq:opt_sparse}
&\argmin_{\mathbf r_{l,k}^*,\gamma_{l,k}} \sum_{\langle \mathcal{O}(c),\mathbf r^* \rangle}
\|\mathbf r^*-\frac{\sum\gamma_{l_k(\mathcal{O}(c)),k} \mathbf r_{l_k(\mathcal{O}(c)),k}^*}{\sum\gamma_{l_k(\mathcal{O}(c)),k}}\|^2+\theta\|\gamma\|_1	\nonumber\\
&\mathbf{s.t.} \ \gamma_{l,k}\in[0.01,1],
\end{align}
which can be solved using an L-BFGS-B optimizer \cite{Zhu:1997:ALF}. 
Finally, both the random forest construction and the controller optimization are integrated into the IL algorithm as outlined in Algorithm~\ref{Alg:mainAlg}.

\setlength{\textfloatsep}{5pt}
\begin{algorithm}[h]
\caption{\label{Alg:mainAlg} Training DOM-controller using imitation learning algorithm.}
\begin{algorithmic}[1]
\renewcommand{\algorithmicrequire}{\textbf{Input:}}
\renewcommand{\algorithmicensure}{\textbf{Output:}}
\Require Initial guess of $\alpha,\beta$, optimal policy $\pi^*$
\Ensure Optimized $\alpha,\beta$
\label{ln:outer}
\LineComment{IL outer loop}
\While{IL has not converged}
\LineComment{Generate training data based on current controller}
\State Sample $\mathcal{D}$ by querying $\pi^*$ as in~\cite{Ross:2011:RIL}
\LineComment{Label each data sample by mean-shift}
\State Run mean-shift on $\{\mathbf r^*\}$ to get $\mathcal{L}$
\LineComment{Extract HOW feature for each data sample}
\For{each $\mathcal{O}(c)$}
\State Extract HOW feature $\mathcal{F}(\mathcal{O}(c))$ as in Section~\ref{sec:featuredesign}
\EndFor
\State Define $\bar{\mathcal{D}}=\{\langle \mathcal{F}(\mathcal{O}(c)),\mathcal{L}(\mathbf r^*) \rangle\}$
\LineComment{Construct random forest, i.e., $\alpha$}
\For{$1\leq k\leq K$}
\State Sample random subset of $\bar{\mathcal{D}}$
\label{ln:rs}
\State Construct $k$-th binary decision tree using~\cite{Shotton:2011:RHP}
\EndFor
\LineComment{Learn a new controller, i.e., determine $\beta$}
\State Solve Equation~\ref{eq:opt_sparse} using L-BFGS-B~\cite{Zhu:1997:ALF}
\EndWhile
\end{algorithmic}
\end{algorithm}

\subsubsection{Analysis}
In typical DOM applications, data are collected using numerical simulations. 
Unfortunately, the high dimensionality of $c$ induces a high computational cost of simulations and generating a large dataset can be quite difficult. 
Therefore, we design our method so that it can be used with small datasets. Our method's performance relies on two parameters: the random forest's stopping criterion (i.e., the threshold of gain in Shannon entropy) and the L1-sparsity of leaf-node confidences, $\theta$. 
We choose to use a large Shannon entropy threshold so that the random forest construction stops early, leaving us with a relatively small number of leaf-nodes, so that $|\beta|$ is small.
We also fix $\theta$ as a very small constant. 

We expect that, with a large enough number of IL iterations, both the number of nodes in each decision tree of the random forest ($\alpha$) and the optimal actions on the leaf-nodes ($\beta$) will converge.
Indeed, such convergence can be guaranteed by the following Lemma. 

\begin{lemma}
\textit{When the number of IL iterations $N\to\infty$, the distribution incurred by the random forest-based controller will converge to a stationary distribution and the expected classification error of the random forest will converge to zero.}
\end{lemma}

\textit{proof:} Assuming that Algorithm~\ref{Alg:mainAlg} generates a controller $\pi^n$ at the $n$-th iteration, Lemma 4.1 of \cite{Ross:2011:RIL} showed that $\pi^n$ incurs a distribution that converges when $n\to\infty$. Obviously, the number of data samples used to train the random forest also increases to $\infty$ with $n\to\infty$. The expected error of a random forest's classification on a stationary distribution converges to zero according to Theorem 5 of \cite{Biau:2012:ARF}. 

In Section~\ref{sec:experiment:forest}, we will further show that, empirically, the number of leaf-nodes in the random forest also converges.
\section{Experiments and Results}
\label{sec:experiment}

We evaluate our proposed DOM framework using one dual-arm robot (ABB Yumi) with 7 degrees-of-freedom in each arm. To get precise information of the 3D object to be manipulated, we set up a vision system including two 3D cameras with different perception fields and precision: one realsense SR300 camera for small objects and one ZR300 camera for large objects. We compute the Cartesian coordinates of the end-effectors of the ABB YuMi as the controlling configuration $r \in \mathbb{R}^6 $ and use an inverse kinematics-based motion planner~\cite{beeson:2015:trac} directly.The entire manipulation system is shown in Figure~\ref{fig:manipulation-system}.

\subsection{Deformable object manipulation using HOW features}

To evaluate the effectiveness of combining the HOW feature with the basic visual seroving manipulation controller, we use $6$ benchmarks (as shown in Figure~\ref{fig:HOW:benchmarks}) with different clothes with different material characteristics and shapes. Moreover, we use different initial and goal states depending on the task, e.g. stretching or folding. 
The details are listed in Table~\ref{tab:benchmarks}.
In these tasks, we use three different forms of goal configurations for the deformable object manipulations, as discussed in Section~\ref{sec:goal}.

For benchmarks 1-3, the task is to manipulate the cloth with human participation. The human is assisted with the task, but the task also introduces external perturbations.
Our approach makes no assumptions about the human motion, and only uses the visual feedback to guess the human's behavior.
In benchmark 1, the robot must anticipate the human's pace and forces for grasping to fold the towel in the air.
In benchmark 2, the robot needs to process a complex task with several goal configurations when performing a folding task.
In benchmark 3, the robot is asked to follow the human's actions to maintain the shape of the sheet. 
For benchmarks 4-6, the task corresponds to manipulating the cloth without human participation and we specify the goal configurations. 
All 6 benchmarks are defined with goal states/features of the cloth, regardless of whether there is a human moving the cloth or not. Because different states of the cloth can be precisely represented and the corresponding controlling parameters can be computed, the robot can perform complicated tasks as well.

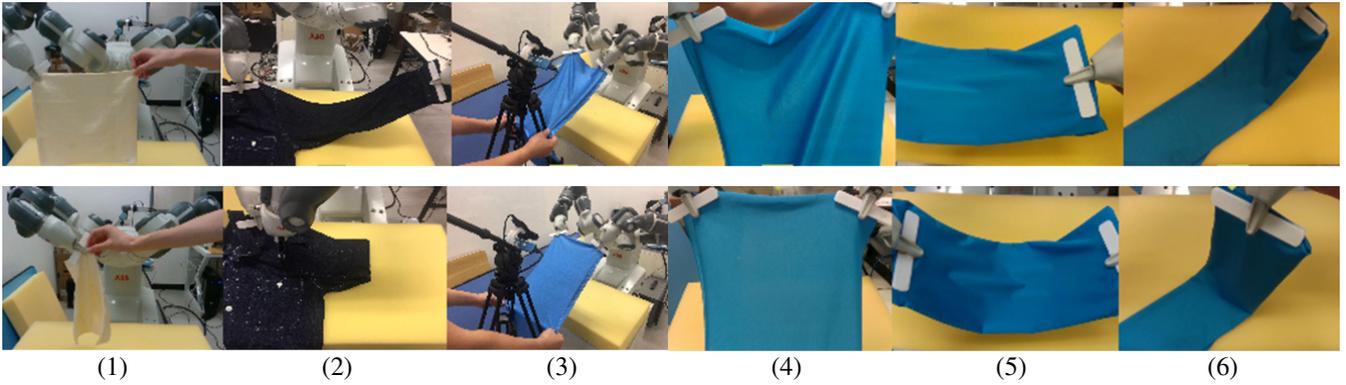
\begin{figure*}[t]
  \centering
  \includegraphics[width=1\textwidth]{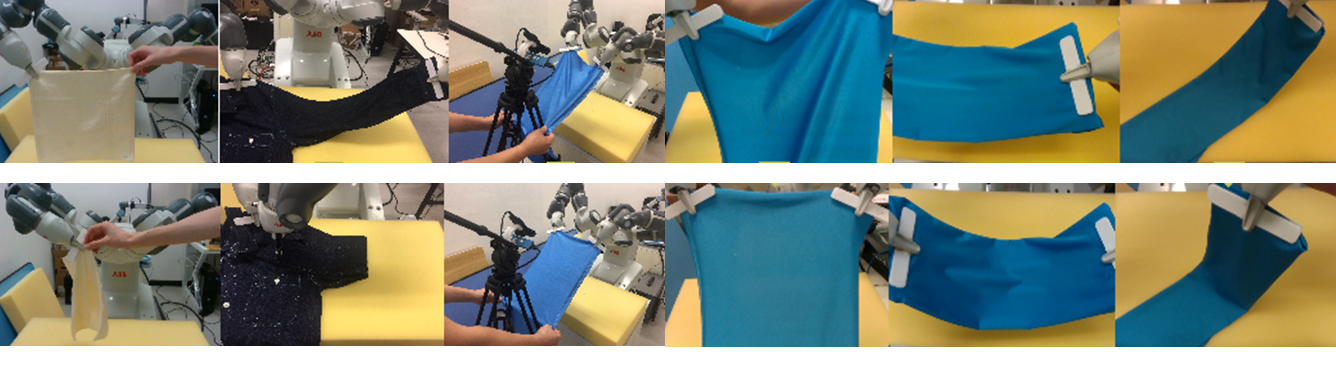}
 \put(-470,0){(1)}
\put(-385,0){(2)}
\put(-300,0){(3)}
\put(-215,0){(4)}
\put(-130,0){(5)}
\put(-50,0){(6)}
  \caption{
  We highlight the performance of using HOW features in DOM tasks on six benchmarks: (1) human-robot
jointly folding a cloth with one hand each, (2) robot folding a cloth with two hands, (3) human-robot stretching a cloth with four
combined hands, (4) flattening, (5) placement, (6) folding. The top row shows the initial state for each task
and the bottom row is the final state. Our approach can handle the perturbations due to human movements.}
  \label{fig:HOW:benchmarks}
\end{figure*}

\begin{table*}[t]
\centering
\begin{tabular}
{llll }
    \hline
     Benchmark\# & Object  & Initial State & Task/Goal \\ \hline
     1 & towel & unfold in the air & fold with human\\ \hline
     2 & shirt & shape (set by human) & fixed shape \\ \hline
     3 & unstretchable cloth & position (set by human) & fixed shape \\ \hline
     4 & stretchable cloth & random  & flattening \\ \hline 
    5 & stretchable cloth & random   & placement \\ \hline
    6 & stretchable cloth & unfolded shape on desk & folded shape \\ \hline
  \end{tabular}
\caption{
	Benchmark tasks for accomplishing DOM tasks using HOW feature as the feedback:
	We highlight various complex manipulation tasks performed using our algorithm. Three of them involve human-robot collaboration and we demonstrate that our method can handle external forces and perturbations applied to the cloth. We use cloth benchmarks of different material characteristics. The initial state is a random configuration or an unfolded cloth on a table, and we specify the final configuration for the task. The benchmark numbers correspond to the numbers shown in Figure~\ref{fig:HOW:benchmarks}.}
\label{tab:benchmarks}
\end{table*}

\subsubsection{Benefits of HOW features}
There are many standard techniques in computer vision and image processing to compute low-dimensional features of deformable models from RGB data. These include standard HOG and color histograms. We evaluate the performance of HOW-features along with the others and explore the combination of these features.
The test involves measuring the success rate of the manipulator in moving towards the goal configuration based on the computed velocity, as given by Equation~\ref{eq:basicfb}.  
We obtain best results in our benchmarks using HOG+HOW features. The HOG features capture the edges in the image and the HOW features capture the wrinkles and deformation, making their benefits complementary. 
For benchmarks 1 and 2, the shapes of the objects change significantly and HOW can easily capture the deformation by enhancing the edges. 
For benchmarks 3, 4, and 5, HOW can capture the deformation by the shadowed areas of wrinkles.
For benchmark 6, the total shadowed area continuously changes during the process, and the color histogram describes the feature slightly better.   

\subsubsection{Benefits of sparse representation}
The main parameter related to the visual feedback dictionary that governs the performance is its size. At runtime, it is also related to the choice of the slack variable in the sparse representation.
As the size of the visual feedback dictionary grows, the velocity error tends to reduce. However, after reaching a certain size, the dictionary contributes less to the control policy mapping. That implies that there is redundancy in the visual feedback dictionary. 

The performance of a sparse representation computation at runtime is governed by the slack variable $\alpha$ in Equations~\ref{alpha1} and~\ref{alpha2}. This parameter provides a tradeoff between data fitting and a sparse solution and governs the velocity error $\|\delta \mathbf r- \delta \mathbf r^*\|_2$ between the desired velocity $\delta \mathbf r^*$ and the actual velocity $\delta \mathbf r$. In practice, $\alpha$ affects the convergence speed. If $\alpha$ is small, the sparse computation has little or no impact and the solution tends to a common linear regression. If $\alpha$ is large, then we suffer from over-fitting.

\begin{figure}[!htp!]
  \centering
  \includegraphics[width=0.4\textwidth]{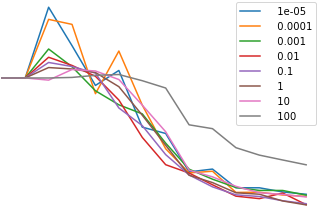}
  \caption{
  Parameter selection for visual feedback dictionary and sparse representation: We vary the dictionary size on the $x$-axis and compute the velocity error for different values of $\alpha$ chosen for sparse representation for benchmark 4.
  }
  \label{fig:HOW:parameters}
\end{figure}

\begin{table*}[!htp!]
\centering
  \begin{tabular}{lccc}
    \hline
    Feature & Benchmark 4 & Benchmark 5  & Benchmark 6  \\  \hline
    HOG & 71.44\%  & 67.31\%  & 82.62\%  \\ \hline 
    Color Histograms & 62.87\% &  53.67\%  & \bf 97.04\%    \\ \hline
    HOW & 92.21\%  &  71.97\% & 85.57\% \\ \hline
    HOW+HOG & \bf 94.53\%  & \bf 84.08\%  & 95.08\%  \\ \hline
  \end{tabular}
  
\caption{Comparison between deformable object features:
We evaluated the success rate of the manipulator based on different features in terms of reaching the goal configuration based on the velocity computed using those features. For each experiment, the number of goals equals the number of frames. There are 393, 204 and 330 frames in benchmarks 4, 5, and 6, respectively. Overall, we obtain the best performance by using HOG + HOW features. 
}
\label{tab:comparison}
\end{table*}

\subsection{Nonlinear DOM controller based on GPR}
\label{sec:experiment:GPR}

For FO-GPR parameters, we set the observation noise $\sigma_n = 0.001$, the RBF spread width $\sigma_{RBF} = 0.6$, the maximum size of the Gram matrix $M = 300$ and $m(\delta \mathbf s) = \mathbf 0$ when comparing with a Jacobian-based method. The execution rate of our approach is $30$ FPS. 

\begin{figure*}[!htp!]
\centering
\begin{subfigure}{0.15\textwidth}
\includegraphics[width=1.0\linewidth]{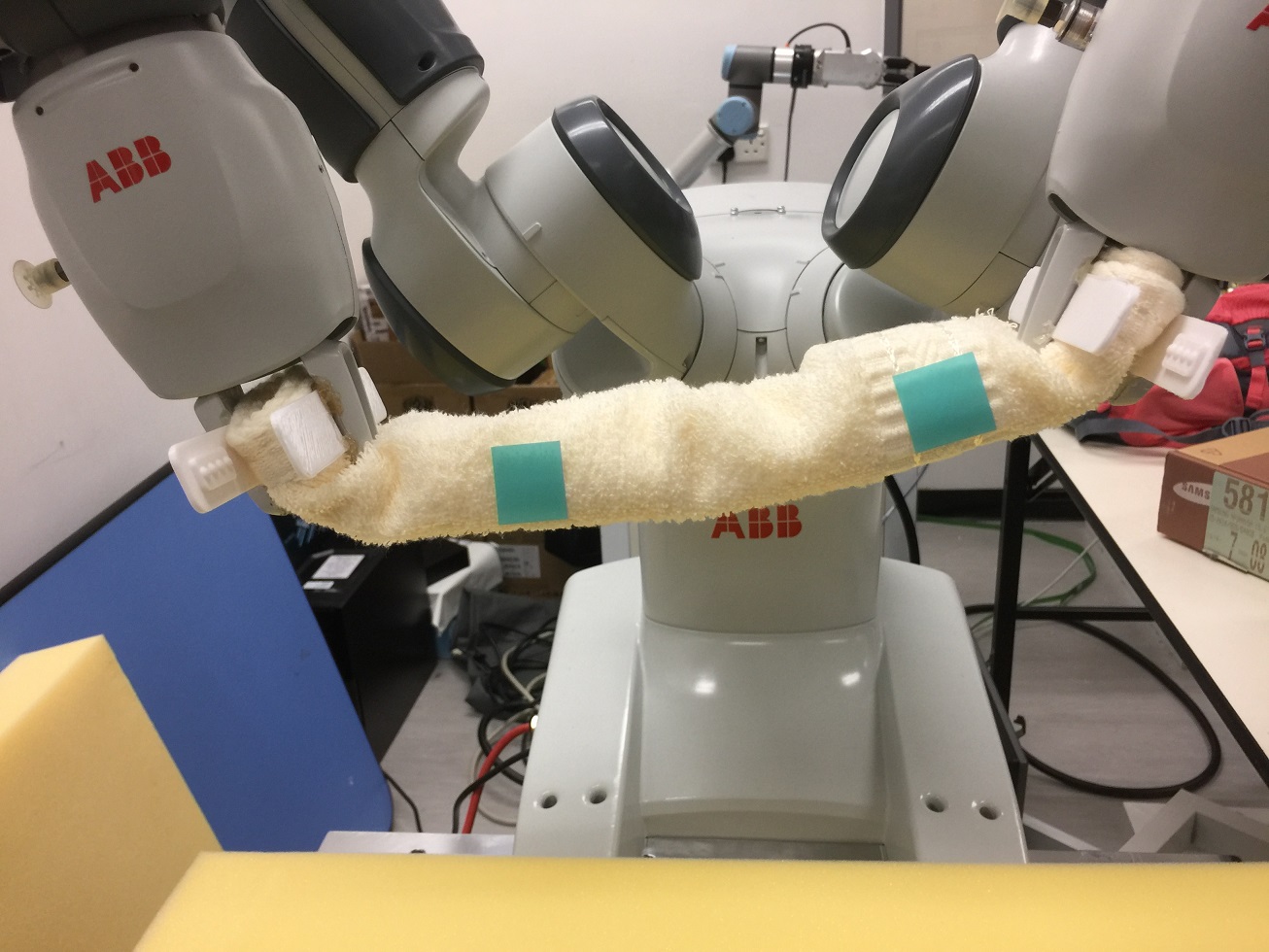}
\hspace{1in}
\end{subfigure}
\begin{subfigure}{0.15\textwidth}
\includegraphics[width=1.0\linewidth]{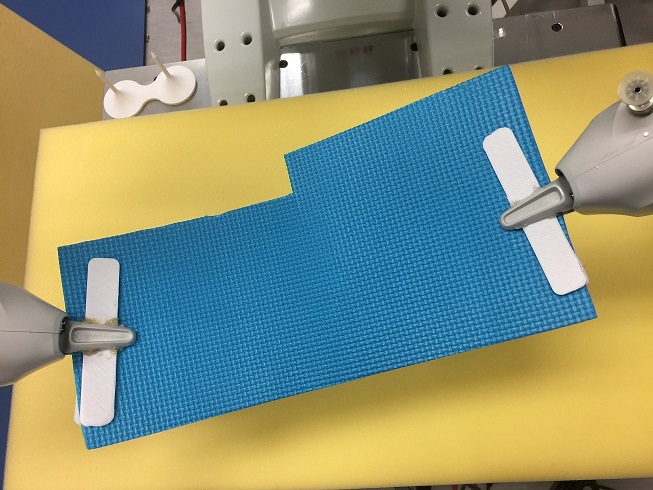}
\hspace{1in}
\end{subfigure}
\begin{subfigure}{0.15\textwidth}
\includegraphics[width=1.0\linewidth]{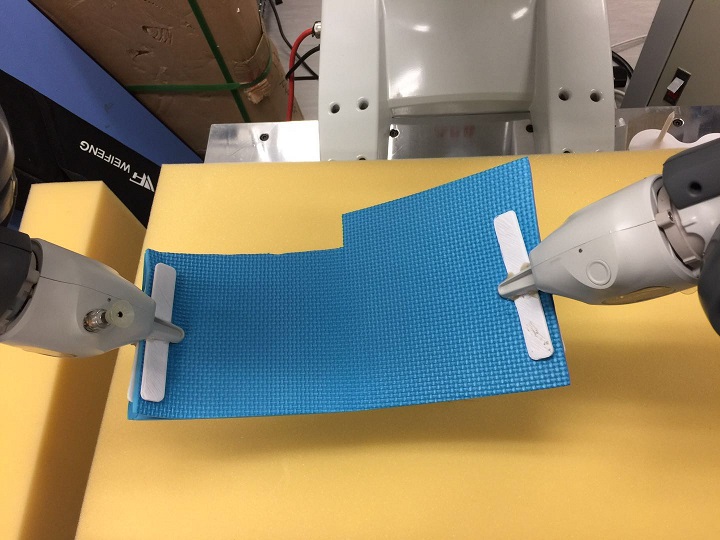}
\hspace{1in}
\end{subfigure}
\begin{subfigure}{0.15\textwidth}
\includegraphics[width=1.0\linewidth]{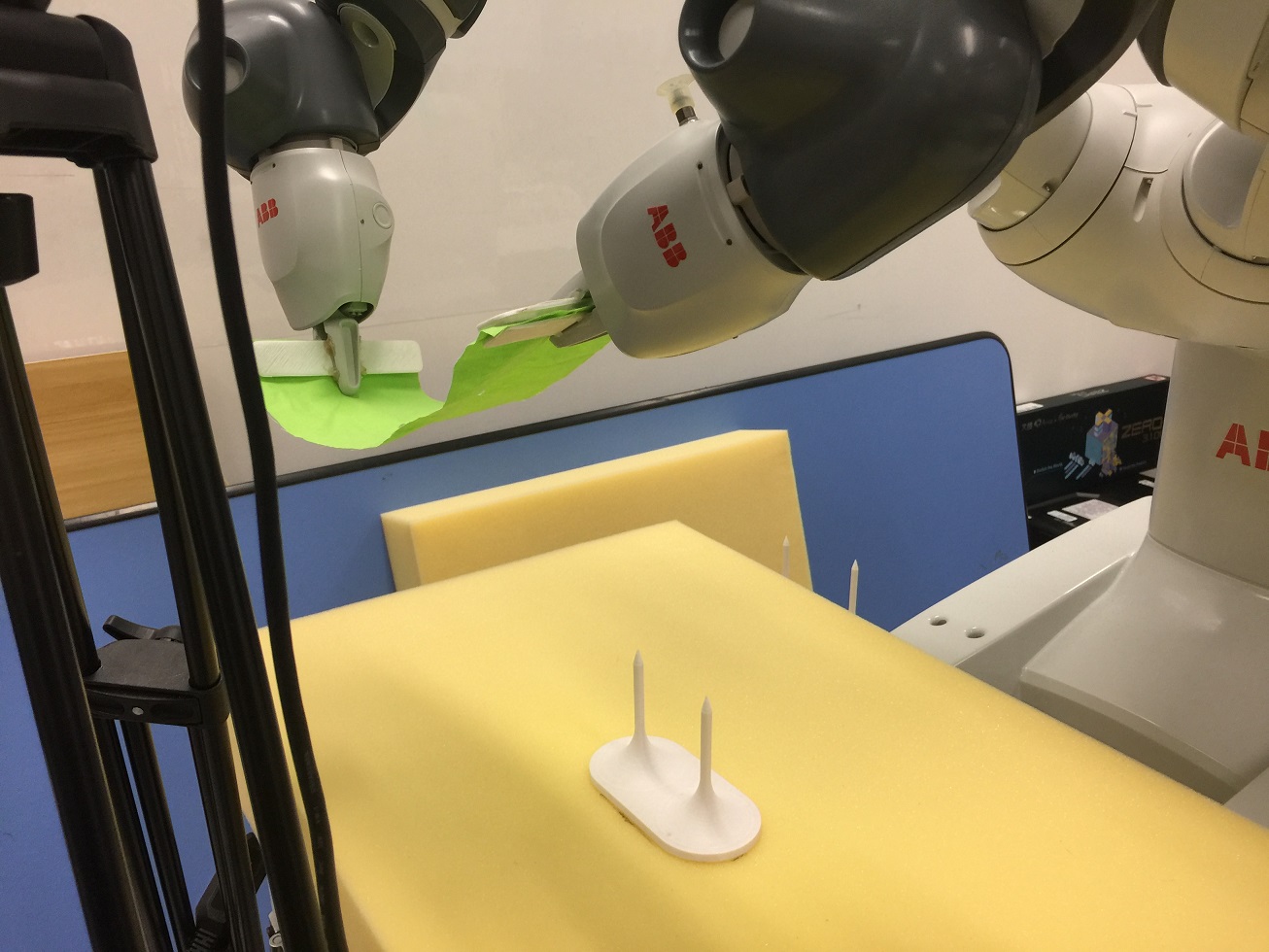}
\hspace{1in}
\end{subfigure}
\begin{subfigure}{0.15\textwidth}
\includegraphics[width=1.0\linewidth]{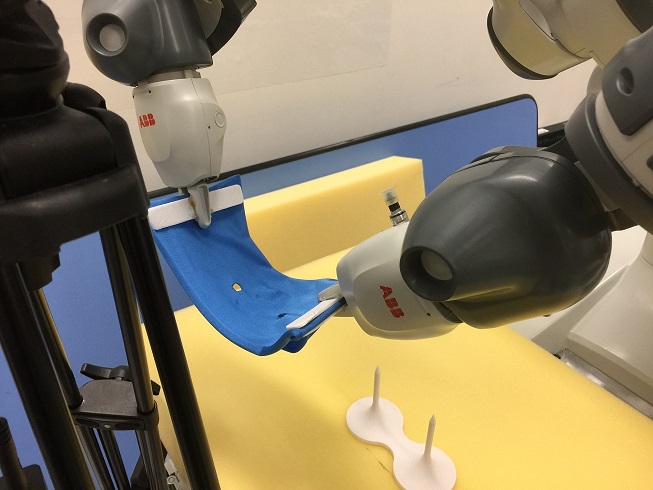}
\hspace{1in}
\end{subfigure}
\begin{subfigure}{0.15\textwidth}
\includegraphics[width=1.0\linewidth]{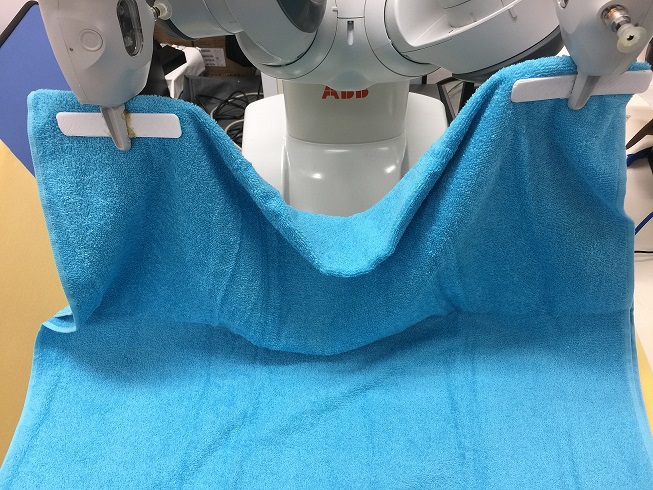}
\hspace{1in}
\end{subfigure}
\begin{subfigure}{0.15\textwidth}
\includegraphics[width=1.0\linewidth]{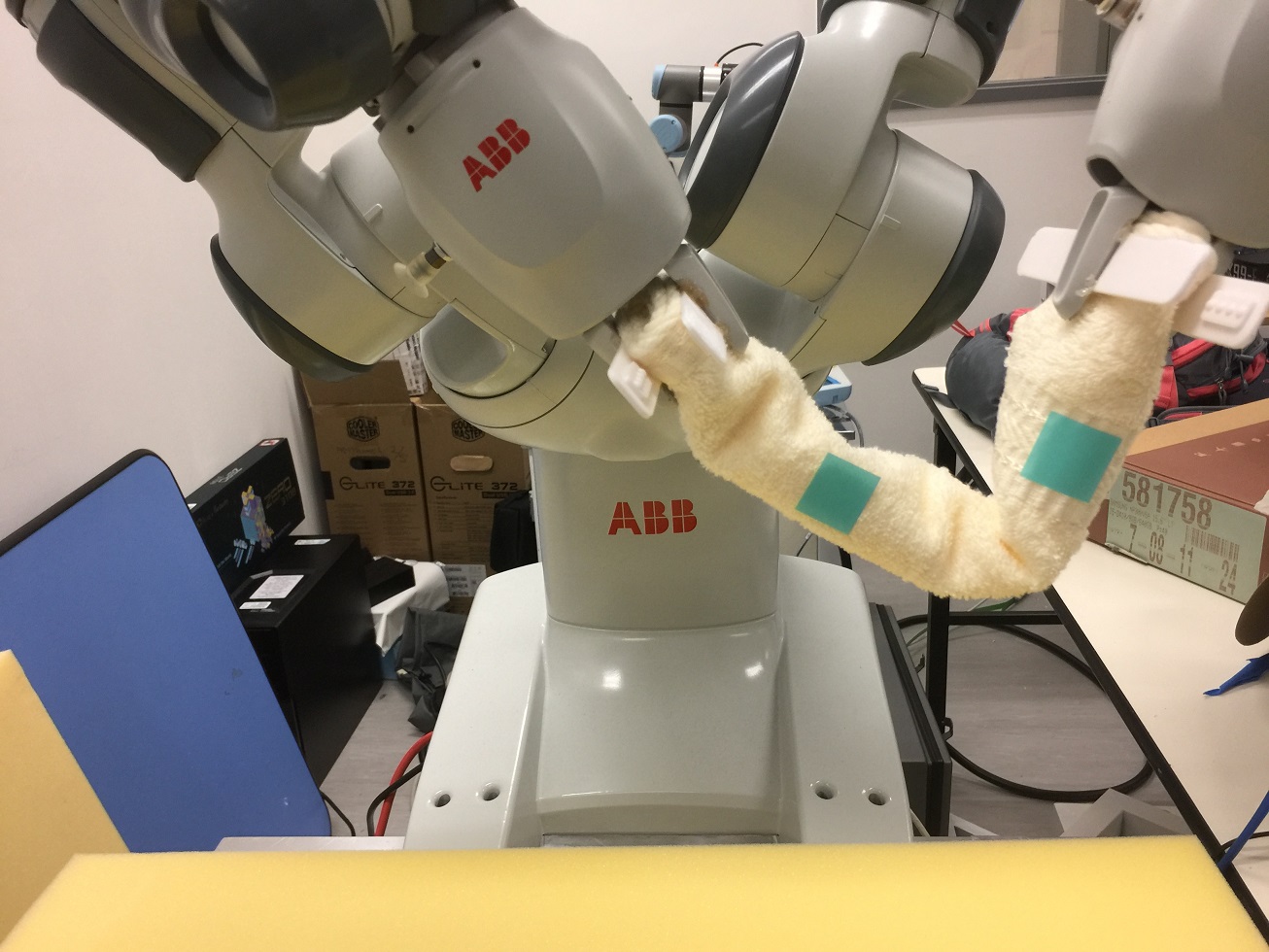}
\caption{}
\label{fig:rolled_towel}
\end{subfigure}
\begin{subfigure}{0.15\textwidth}
\includegraphics[width=1.0\linewidth]{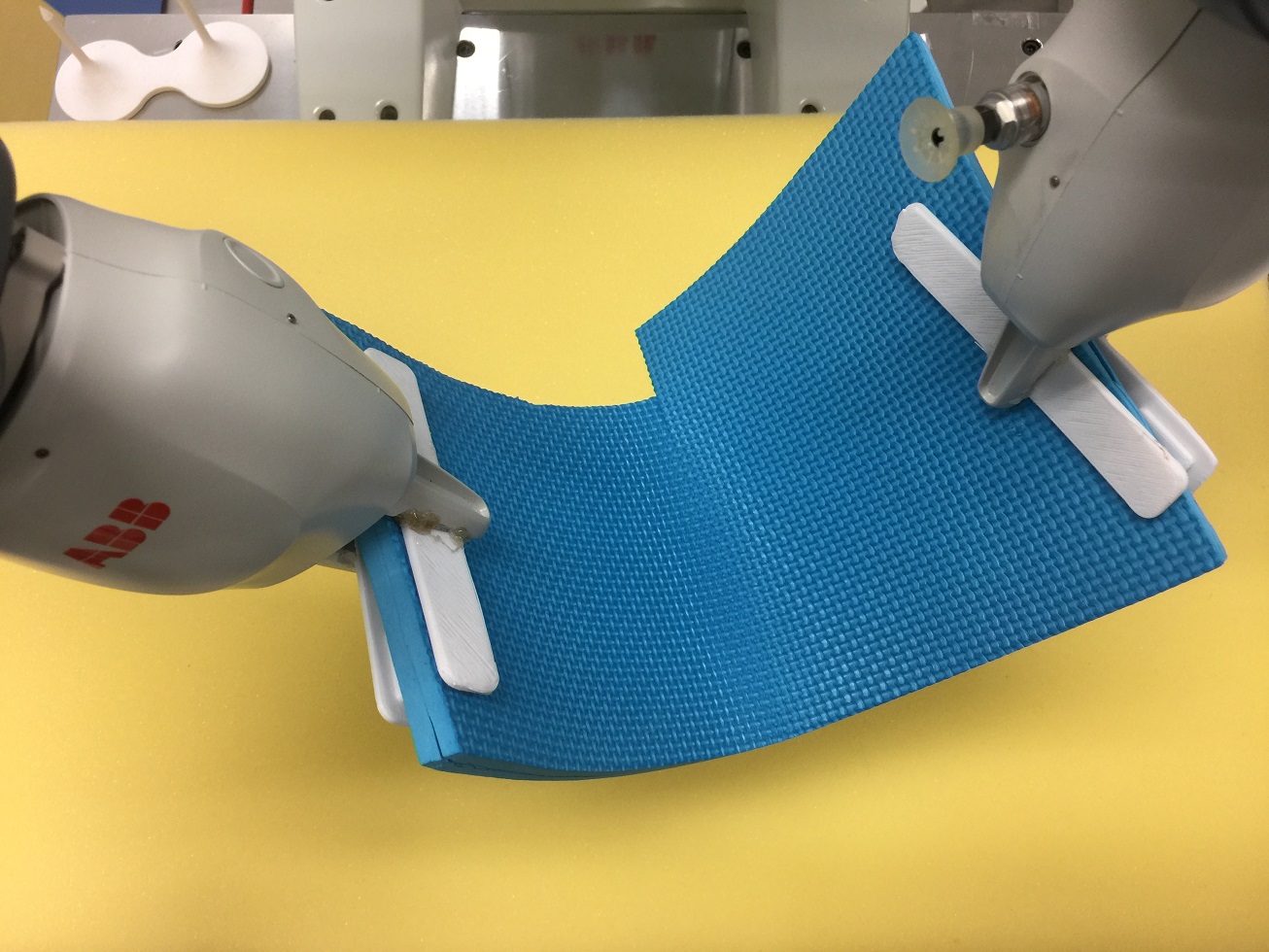}
\caption{}
\label{fig:plastic_plate}
\end{subfigure}
\begin{subfigure}{0.15\textwidth}
\includegraphics[width=1.0\linewidth]{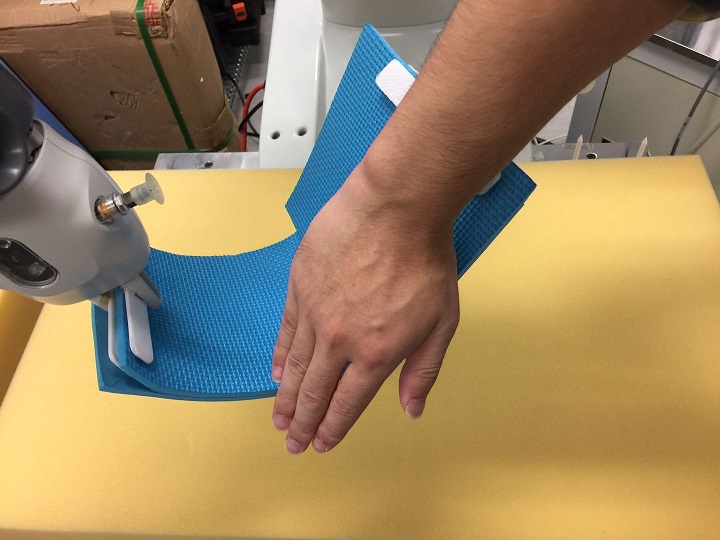}
\caption{}
\label{fig:plastic_plate_occlussion}
\end{subfigure}
\begin{subfigure}{0.15\textwidth}
\includegraphics[width=1.0\linewidth]{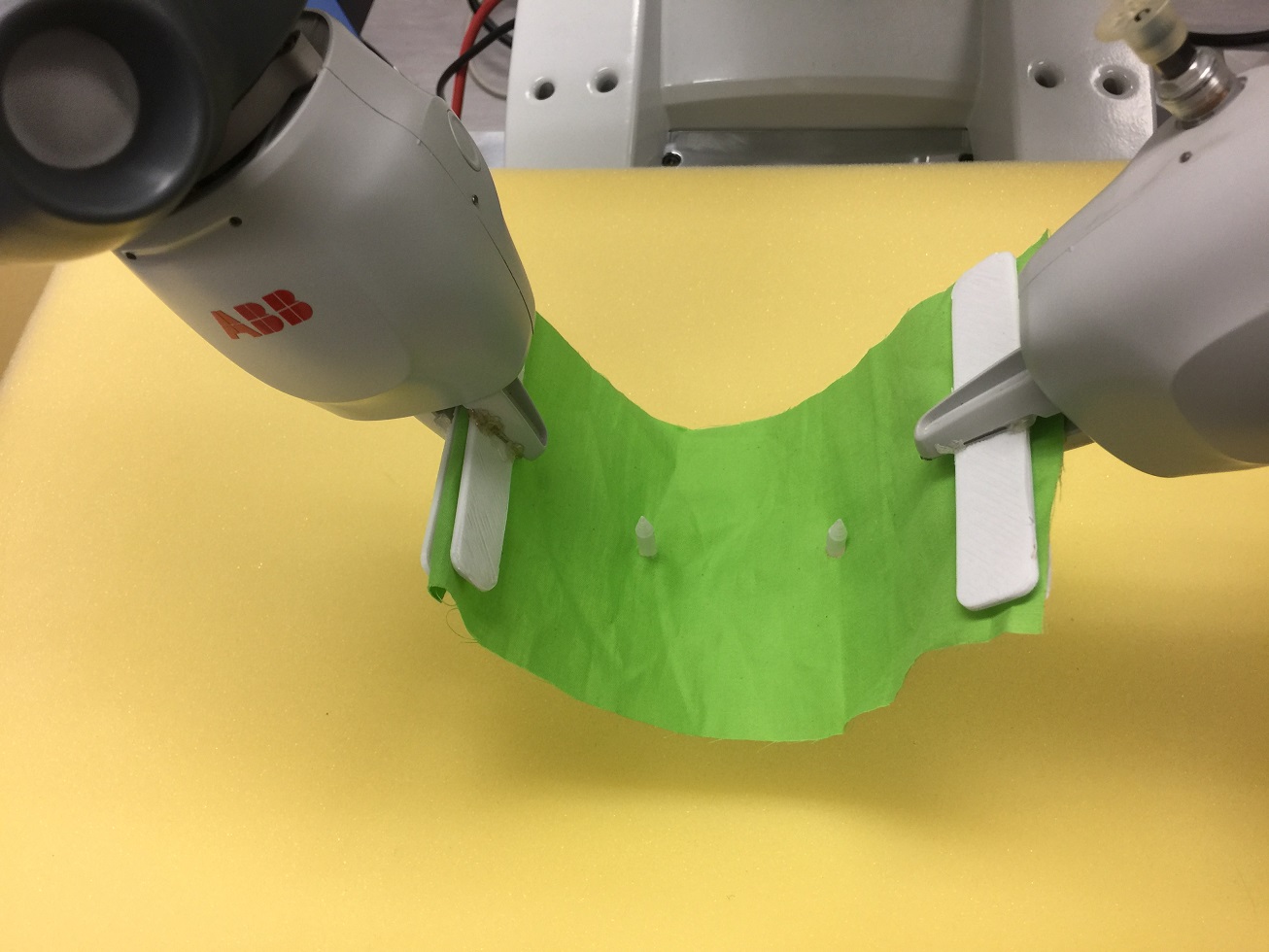}
\caption{}
\label{fig:stiff_fabric}
\end{subfigure}
\begin{subfigure}{0.15\textwidth}
\includegraphics[width=1.0\linewidth]{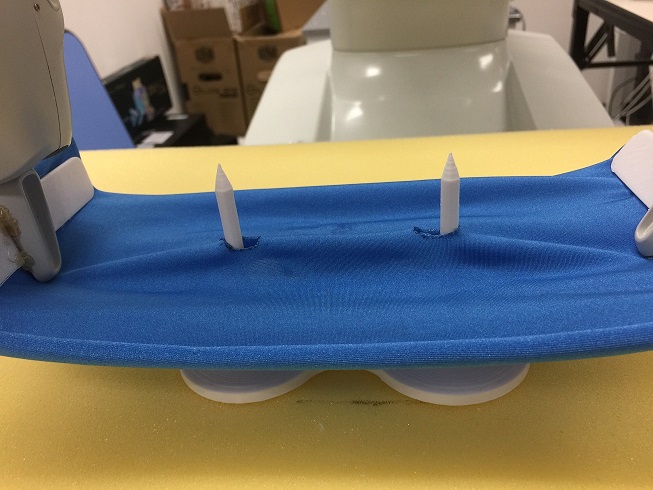}
\caption{}
\label{fig:elastic_fabric}
\end{subfigure}
\begin{subfigure}{0.15\textwidth}
\includegraphics[width=1.0\linewidth]{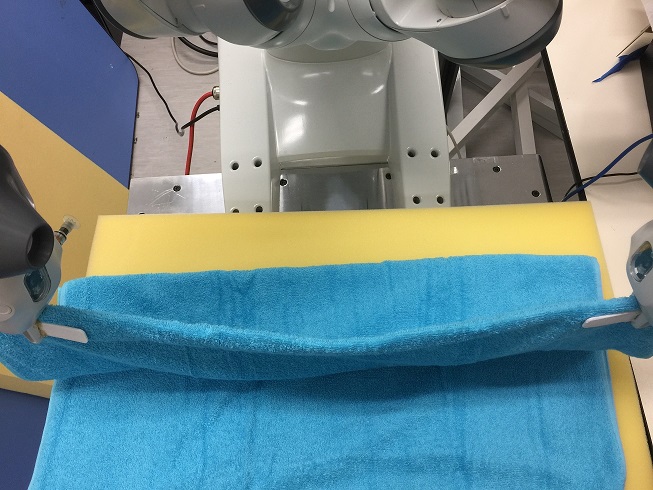}
\caption{}
\label{fig:bath_towel}
\end{subfigure}
\caption{The set of tasks used to evaluate the performance of the GPR-based nonlinear DOM controller: (a) rolled towel bending, (b) plastic sheet bending, (c) plastic sheet bending with occlusion, (d) peg-in-hole for unstretchable fabric, (e) peg-in-hole for stretchable fabric, and (f) towel folding. The first row shows the initial state of each object before the manipulation and the second row shows the goal states of the object after the successful manipulation.}
\label{fig:multi_tasks}
\end{figure*}

To evaluate the performance of the GPR-based nonlinear DOM controller, we design different manipulation tasks involving distinct objects as shown in Figure~\ref{fig:multi_tasks}. In these tasks, the feedback points can either be all object points observed by the camera or several marked points (similar to what is done in~\cite{Alarcon:2016:ADM}) critical for accomplishing a task.
\begin{itemize}
\item{\emph{Rolled towel bending}:} This task aims at bending a rolled towel in to a specific goal shape as shown in Figure~\ref{fig:rolled_towel}. We use a $4$-dimension feature vector $\mathbf s=\left[\mathbf c, d\right]$ as the feature vector in the FO-GPR driven visual-servo, where $\mathbf c$ is the centroid feature described by Equation~\ref{eq:centroid} and $d$ is the distance feature as described by Equation~\ref{eq:distance} for two feature points on the towel. 
\item{\emph{Plastic sheet bending}:}  The goal of this task is to manipulate a plastic sheet into a preassigned curved status as shown in Figure~\ref{fig:plastic_plate}. We use a $4$-dimension feature vector $\mathbf s=\left[\mathbf c, \sigma\right]$ to describe the state of the plastic sheet, where $\mathbf c$ is the centroid feature described by Equation~\ref{eq:centroid} and $\sigma$ is the surface variation feature computed by Equation~\ref{eq:variation}. In Figure~\ref{fig:plastic_plate_occlussion}, we also perform a task where the sheet is partially observed during the manipulation to demonstrate the robustness and stability of our proposed controller.
\item{\emph{Peg-in-hole for fabrics}:} This task aims at moving cloth pieces so that the pins can be inserted into the corresponding holes in the fabric. Two different types of fabric with different stiffnesses have been tested in our experiment: one is an unstretchable fabric, as shown in Figure~\ref{fig:stiff_fabric}, and the other is a stretchable fabric, as shown in Figure~\ref{fig:elastic_fabric}. The $6$-dimension feature vector $\mathbf s = \mathbf \rho$ is the position of feedback points as described in Equation~\ref{eq:position}. 
\item{\emph{Towel folding}:} This task aims at flattening and folding a towel into a desired state, as shown in Figure~\ref{fig:bath_towel}. We use a binned histogram of extended FPFH, which generates a feature vector of $135$ dimensions, to describe the towel's shape. Since the feature has a large dimension, for this experiment we need to manually move the robot in the beginning to explore sufficient data so that the FO-GPR can learn an adequate initial model for the complex deformation function. 
\end{itemize}

\begin{figure}[!htp!] 
\centering
\begin{subfigure}{0.48\linewidth}
\includegraphics[width=1.0\linewidth]{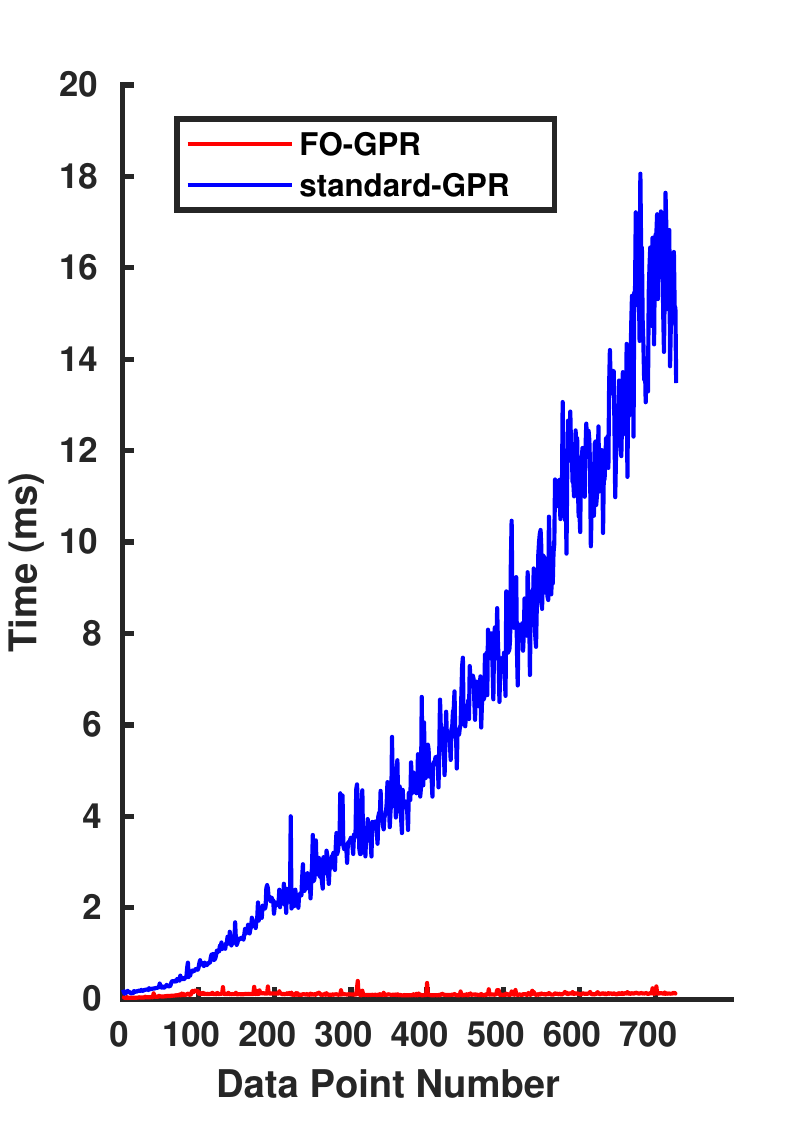}
\caption{}
\label{fig:idle-time}
\end{subfigure}
\begin{subfigure}{0.48\linewidth}
\includegraphics[width=1.0\linewidth]{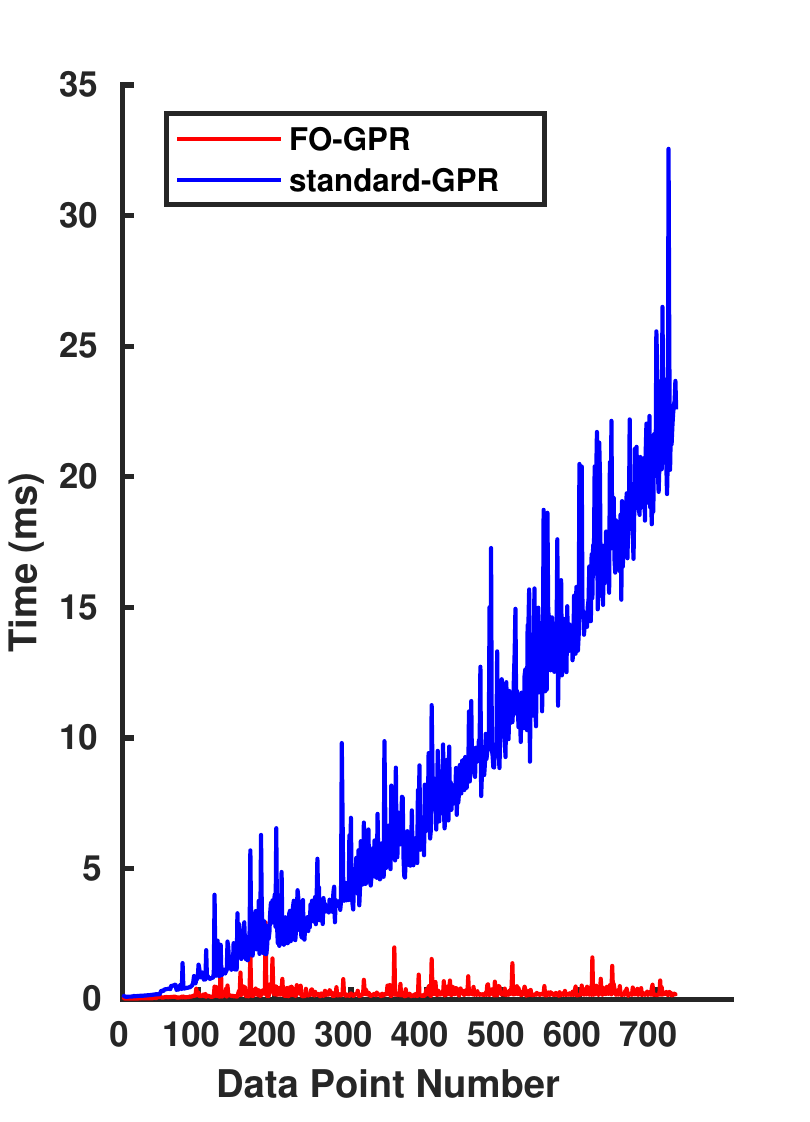}
\caption{}
\label{fig:busy-time}
\end{subfigure}
\caption{Comparison of the time cost of FO-GPR and standard GPR: (a) the time cost comparison between GP model estimation; (b) the time cost comparison for the entire manipulation process.}
\end{figure}

\begin{figure}[!htp!] 
\centering
\includegraphics[width=1\linewidth]{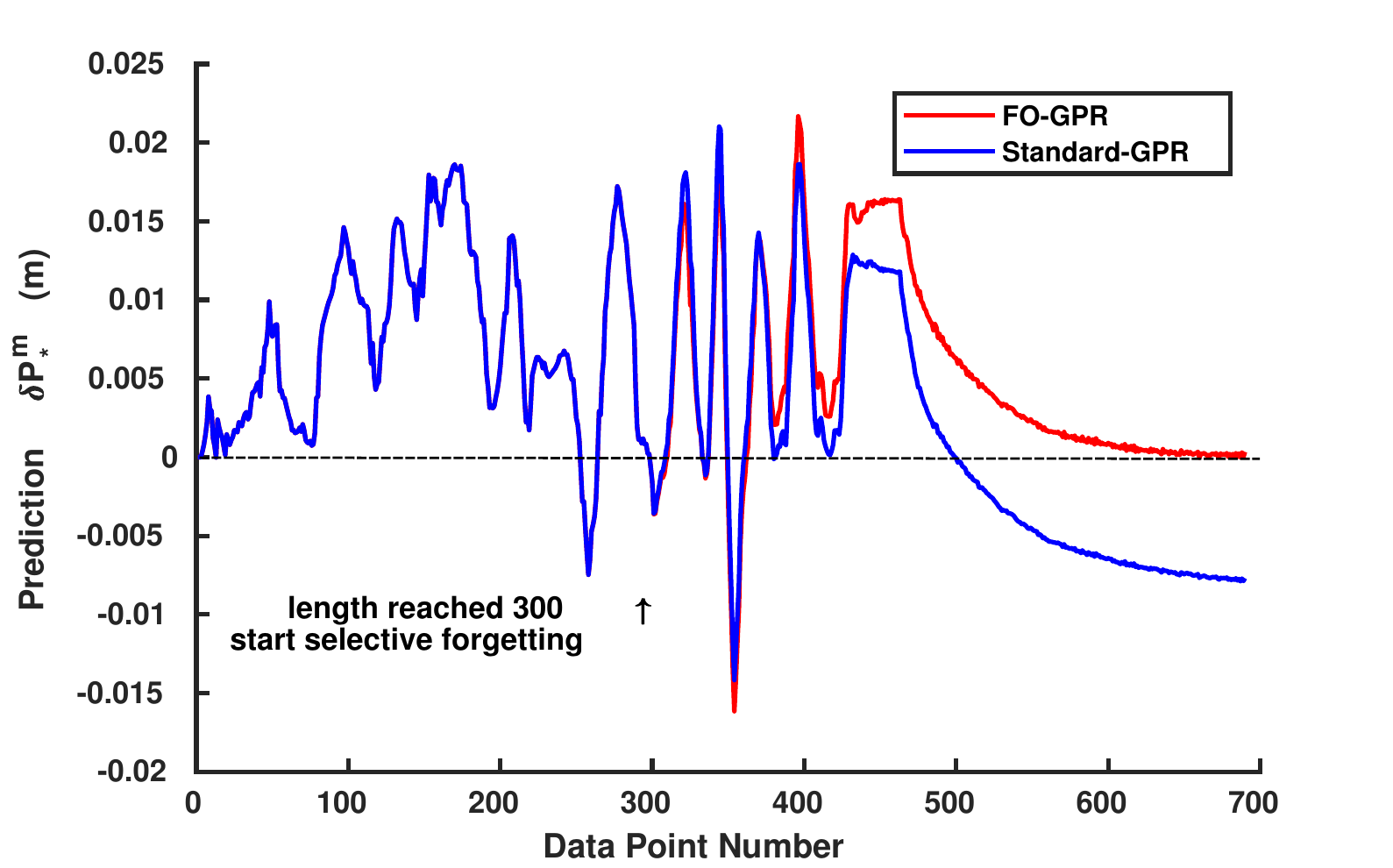}
\caption{Impact of selective forgetting in FO-GPR: FO-GPR is superior to standard GPR in terms of the computational cost and the deformation model accuracy.
}
\label{fig:information}
\end{figure}

\begin{figure}[!htp!] 
\centering
\includegraphics[width=1\linewidth]{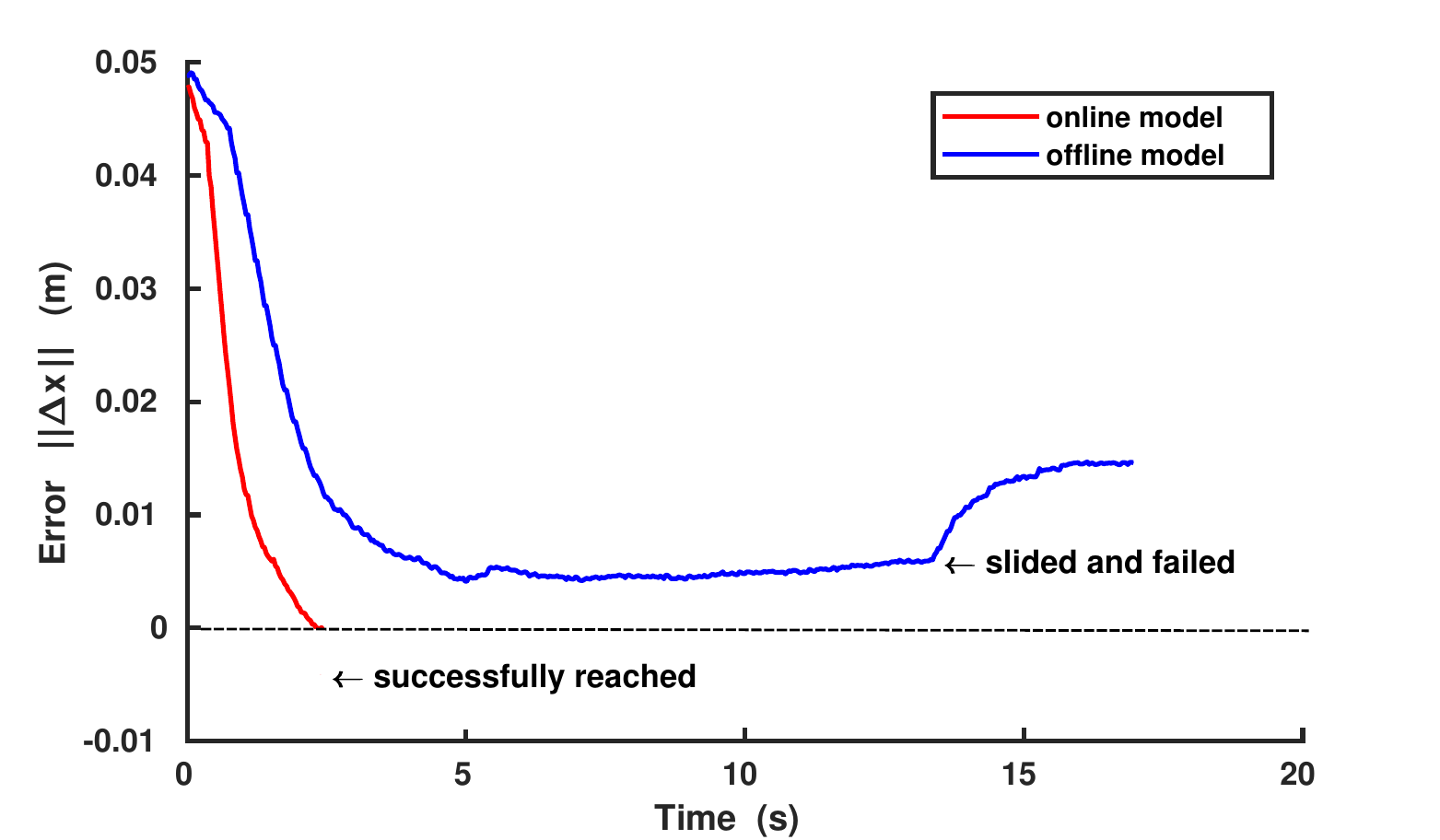}
\caption{Comparison of online and offline GP models for rolled towel manipulation. The controller based on the online model succeeds while the controller based on the offline model fails.}
\label{fig:online_offline}
\end{figure}

\begin{figure}[!htp!] 
\centering
\begin{subfigure}{0.5\textwidth}
\includegraphics[width=1.0\linewidth]{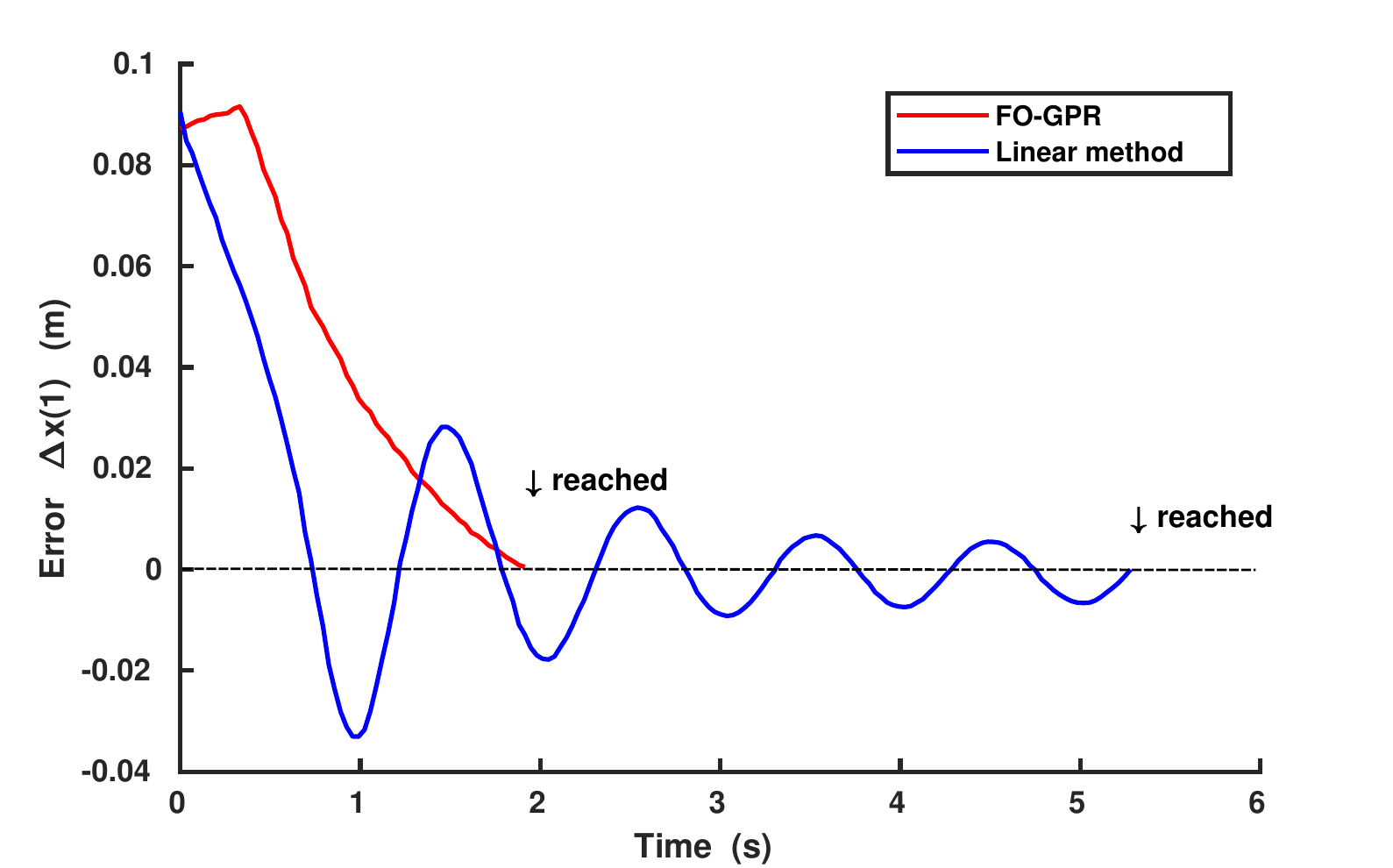}
\caption{}
\label{fig:FO_Linear4d1}
\end{subfigure}
\begin{subfigure}{0.5\textwidth}
\includegraphics[width=1.0\linewidth]{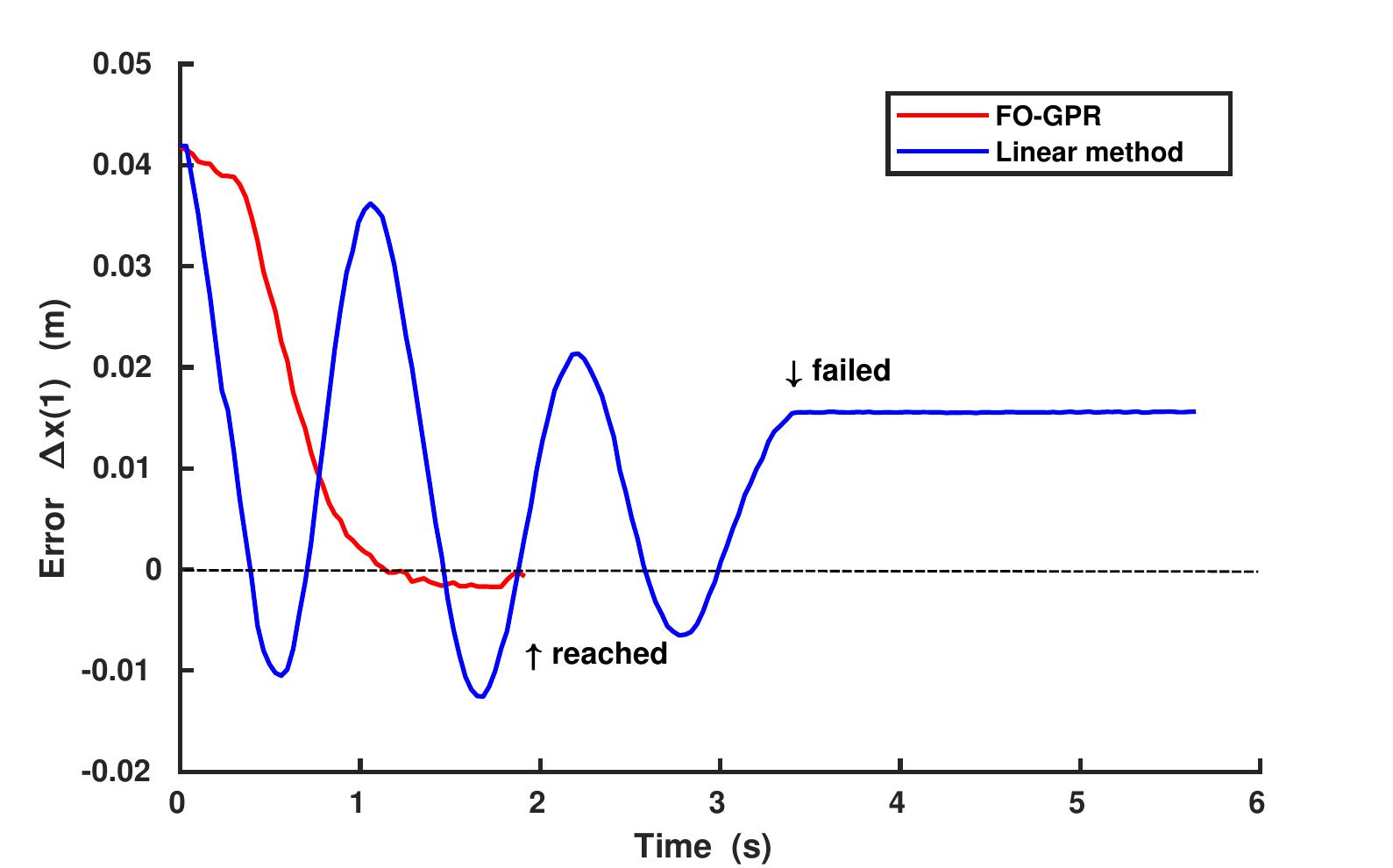}
\caption{}
\label{fig:FO_Linear6d}
\end{subfigure}
\caption{Comparison of the controllers based on FO-GPR and the linear model: (a) comparison on the rolled towel task; (b) comparison on the peg-in-hole task with stretchable fabric.}
\end{figure}

\begin{figure}[!htp!] 
\centering
\begin{subfigure}{0.5\textwidth}
\includegraphics[width=1.0\linewidth]{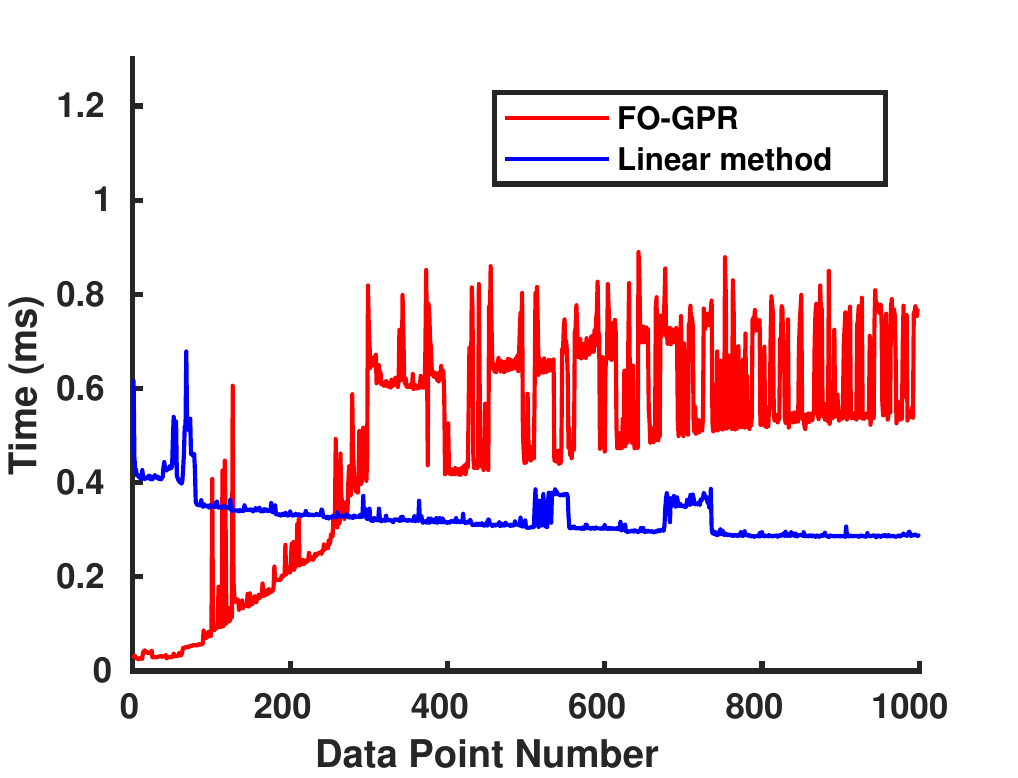}
\caption{}
\label{fig:time1_4}
\end{subfigure}
\begin{subfigure}{0.5\textwidth}
\includegraphics[width=1.0\linewidth]{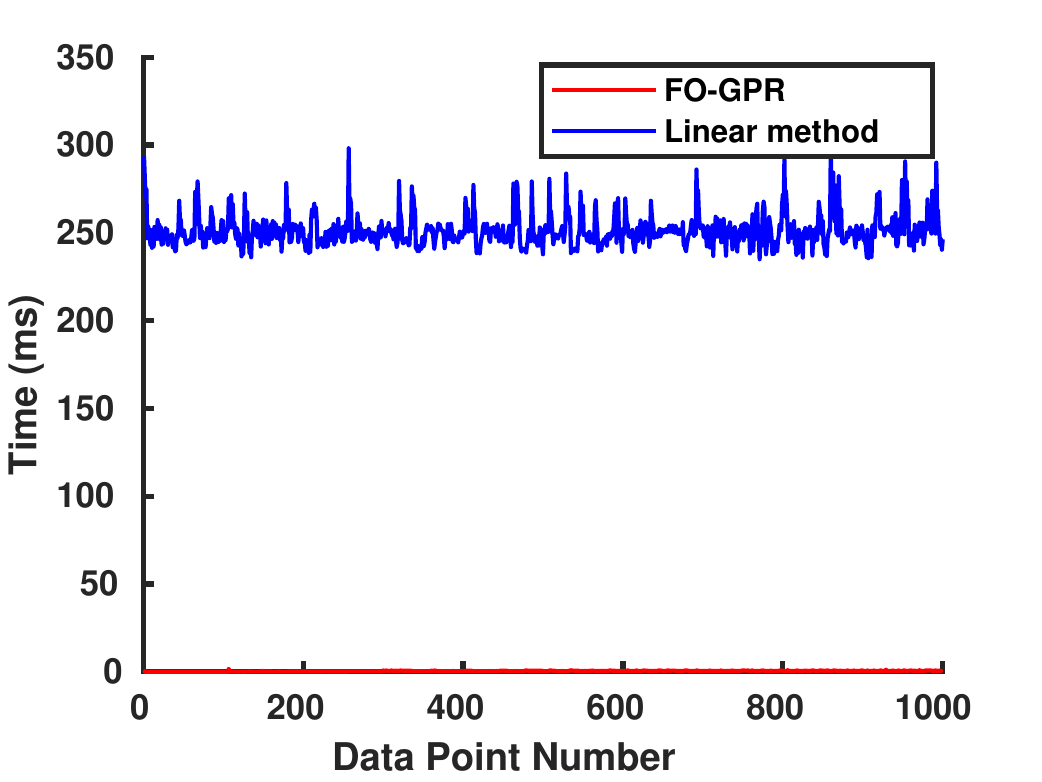}
\caption{}
\label{fig:time2_1000}
\end{subfigure}
\caption{Comparison of the time cost of FO-GPR and linear model on the rolled towel task (a) and the plastic sheet bending task (b).}
\label{fig:time-com-fo-j}
\end{figure}

Our FO-GPR based manipulation control can accomplish all the five tasks efficiently and accurately. The success rate is around 70\% for the peg-in-hole task and 90\% for the other tasks. 

Next, we provide some quantitative analysis of our approach by comparing it with some state-of-the-art approaches. 

\subsubsection{Comparison of computational cost with standard GPR}
As shown in Figure~\ref{fig:idle-time}, the time cost of the standard GPR operation in each iteration increases significantly when the number of training points increases, which makes the online deformation function update impossible. Our FO-GPR method's time cost is always under $2$ ms. We also compare the time cost of each complete cycle of the manipulation process, including feature extraction, tracking, robot control, and GPR, and the result is shown in Figure~\ref{fig:busy-time}. Again, the time cost of manipulation using standard GPR fluctuates significantly, and it can be 10 times slower than our FO-GPR based manipulation, whose time cost is always below $5$ ms and allows for real-time manipulation. This experiment is performed using the rolled towel bending task.

\subsubsection{Impact of selective forgetting}
In Figure~\ref{fig:information}, we compare the GP prediction quality between FO-GPR and the standard GPR on the rolled towel task, in order to show the impact of selective forgetting in FO-GPR. We record about $700$ data entries continuously. The first $450$ entries are produced using a random controller, and the rest are generated by the FO-GPR based controller, which drives the soft object toward the target state smoothly. 
The controllers using two GPR models provide the same velocity output before the data size reaches the maximum limit $M = 300$. From this point on, FO-GPR selectively forgets uninformative data while the standard GPR still uses all the data for prediction. For data points with indices between $300$ and $450$, the outputs from two controllers are similar, which implies that FO-GPR still provides a sufficiently accurate model. After that, the FO-GPR based controller drives the object toward the goal and the controller output eventually becomes zero. However, for standard GPR, the controller output remains not zero. This experiment suggests that the performance of FO-GPR is much better than GPR in real applications in terms of both time saving and the accuracy of the learned deformation model. 

\subsubsection{Comparison of online and offline GPR} In this experiment, we leave the Gram matrix unchanged after a while in the rolled towel manipulation task and compare the performance of the resulting offline model with that of our online learning approach. As shown in Figure~\ref{fig:online_offline}, the error in the feature space $\|\Delta \mathbf s\|_2$ decreases at the beginning of manipulation while using both models for controls. However, when the soft object is close to its target configuration, the controller using the offline model cannot output accurate prediction due to the lack of data around the unexplored target state. Thanks to the balance of exploration and exploitation of online FO-GPR, our method updates the deformation model all the time and thus can output a relatively accurate prediction so that the manipulation process is successful.

\subsubsection{Comparison of FO-GPR and adaptive Jacobian-based method}
We compare our approach to the state-of-the-art adaptive Jacobian-based method for soft object manipulation~\cite{Alarcon:2016:ADM}, which learns the Jacobian matrix using a squared error function and stochastic gradient descent (SGD). This method can be considered as a special case of our approach, where the deformation function is restricted to a linear function. First, through experimentation, we find that the learning rate of the linear model has a great impact on the manipulation performance and needs to be tuned offline for different tasks. Our approach can use the same set of parameters for all tasks. Next, we perform both methods on the rolled towel and the peg-in-hole with stretchable fabric tasks, and the results are shown in Figures~\ref{fig:FO_Linear4d1} and~\ref{fig:FO_Linear6d}, respectively. 
To visualize the comparison results, we choose one dimension from the feature vector $\mathbf s$ and plot it. In Figure~\ref{fig:FO_Linear4d1}, we observe that the error of the controller based on the linear model decreases quickly, but due to the error in other dimensions the controller still outputs a high control velocity and thus vibrations begin. The controller needs a long time to accomplish the task. In contrast, the error of the plotted dimension decreases slower but the controller finishes the task faster because the error of all dimensions declines to zero quickly, thanks to the nonlinear modeling capability of GPR. In the peg-in-hole task, we can observe that the GPR-based controller successfully accomplishes the task while the controller based on the linear model fails. We also compare the computational time of both methods in Figure~\ref{fig:time-com-fo-j}. For the rolled towel task, the adaptive-Jacobian based method is slightly (about 0.4 ms) faster than our approach since there are only 2 feedback points per frame. For the plastic sheet bending task with around 1200 feedback points per frame, the adaptive-Jacobian-based method is significantly slower due to the expensive matrix multiplication while computing the gradient.

\subsubsection{Robustness to occlusion}
Our method is robust to moderate levels of occlusions, as shown in Figure~\ref{fig:plastic_plate_occlussion} and~\ref{fig:bath_towel}. Thanks to the online learning mechanism, our controller can output a correct prediction for an occluded point once it has already been observed before and thus has been learned by the GPR model. Nevertheless, our method may not provide correct prediction and the controller may fail if the occluded part changes significantly and the features change rapidly.

\subsection{Random forest-based DOM controller}
\label{sec:experiment:forest}

Here we demonstrate the performance of the random forest-based DOM controller on both a simulated environment and real robot hardware. 
For the simulated environment, the robot's kinematics are simulated using Gazebo~\cite{Koenig:2004:DUP} and the cloth dynamics are simulated using ArcSim~\cite{Narain:2012:AAR}, a high-accuracy cloth simulator. We use OpenGL to capture RGB-D in this simulated environment.
For the real robotic environment, we use the same setup as before i.e., a RealSense depth camera to capture 640x480 RGB-D images and a 14-DOF ABB YuMi dual-armed manipulator to perform actions. Here the deformable object to be manipulated is restricted to a $35$cm$\times$$30$cm rectangular-shaped piece of cloth, as shown in Figure~\ref{fig:forest:task}.

\begin{figure}[ht]
  \centering
  \includegraphics[width=0.4\textwidth]{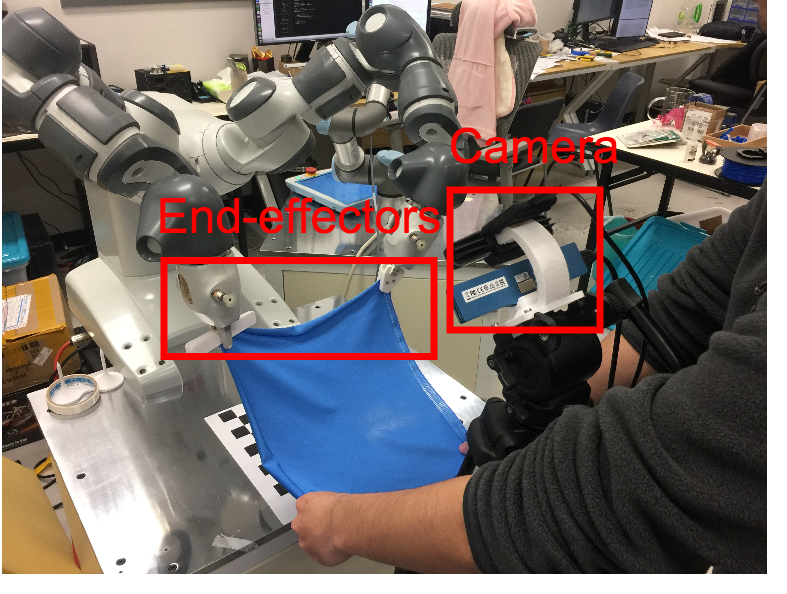}
  \caption{
    Setup for random forest-based DOM controller: A dual-armed robot and a human are holding four corners of the cloth. We use a $14$-DOF dual-armed ABB YuMi and a RealSense RGB-D camera to perform complex manipulation tasks. Our goal is to manipulate a $35$cm$\times$$30$cm rectangular-shaped piece of cloth. }
  \label{fig:forest:task}
\end{figure}

Our goal is to manipulate a rectangular piece of cloth with four corners initially located at: $v^0=(0,0,0), v^1=(0.3,0,0), v^2=(0,0.35,0), v^3=(0.3,0.35,0)$(m). 
Our manipulator holds the first two corners, $v^0,v^1$, of the cloth and the environmental uncertainty is modeled by having a human holding the last two corners, $v^2,v^3$, of the cloth, so that we have $\mathbf r \triangleq\TWOC{v^0}{v^1}$ and each control action is $6$-dimensional. The human could move $v^2,v^3$ to an arbitrary location under the following constraints:
\begin{align}
&0\text{m}<\|v^2-v^3\|\leq 0.3\text{m}	\\
&\|\TWOC{v^2}{v^3}_{i+1}-\TWOC{v^2}{v^3}_i\|_\infty<0.1\text{(m/s)},
\end{align}
where the first constraint avoids tearing the cloth apart and the second constraint ensures that the speed of the human hand is slow. 

\subsubsection{Synthetic benchmarks}
\label{sec:sim}
\begin{figure}[ht]
\includegraphics[width=0.5\textwidth]{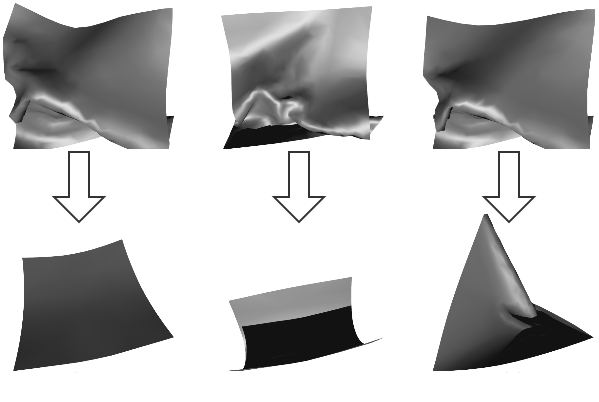}
\put(-225,3){(a)}
\put(-130,3){(b)}
\put(-40,3){(c)}
  \caption{
   Synthetic benchmarks: We highlight the real-time performance of the random-forest based controller on three tasks for the robot in simulator (a) keep the cloth straight, (b) keep the cloth bent, (c) keep the cloth twisted.}
  \label{fig:forest:syntask}
\end{figure}

To evaluate the robustness of our method, we design three manipulation tasks listed below:
\begin{itemize}
\item Cloth should remain straight in the direction orthogonal to the human hands. This is illustrated in Figure~\ref{fig:forest:syntask} (a). Given $v^2,v^3$, the robot's end-effector should move to:
\begin{align}
v^0&=v^2+0.35\frac{z\times(v^3-v^2)}{\|z\times(v^3-v^2)\|},\\
v^1&=v^3+0.35\frac{z\times(v^3-v^2)}{\|z\times(v^3-v^2)\|}.
\end{align}
\item Cloth should remain bent in the direction orthogonal to the human hands. This is illustrated in Figure~\ref{fig:forest:syntask} (b). Given $v^2,v^3$, the robot's end-effector should move to:
\begin{align}
v^0&=v^2+0.175\frac{z\times(v^3-v^2)}{\|z\times(v^3-v^2)\|},\\
v^1&=v^3+0.175\frac{z\times(v^3-v^2)}{\|z\times(v^3-v^2)\|}.
\end{align}
\item Cloth should remain twisted along the direction orthogonal to the human hands. This is illustrated in Figure~\ref{fig:forest:syntask} (c). Given $v^2,v^3$, the robot's end-effector should move to:
\begin{align}
v^0&=v^3+0.175\frac{z\times(v^3-v^2)}{\|z\times(v^3-v^2)\|},\\
v^1&=v^2+0.175\frac{z\times(v^3-v^2)}{\|z\times(v^3-v^2)\|}.
\end{align}
\end{itemize}
The above formula for determining $v^0,v^1$ is used to simulate an expert. These three equations assume that the expert knows the location of the human hands, but that the robot does not have this information and it must infer this latent information from a single-view RGB-D image of the current cloth configuration.

\subsubsection{Transferring from a simulated environment to real robot hardware}
\label{sec:real}
To transfer the controller parametrization trained on the synthetic data, we complete the following procedures.
First, we use the camera calibration techniques to get both the extrinsic and intrinsic matrix of the RealSense camera. 
Second, we compute the camera position, camera orientation, and the clipping range of the simulator from the extracted parameters.
Third, we generate a synthetic depth map using these parameters and train three tasks described in Section~\ref{sec:sim} using the random forest-based controller parametrization along with the IL algorithm.
After that, we integrate the resulting controller with the ABB YuMi dual-armed robot and the RealSense camera via the ROS platform.
As shown in Figure~\ref{fig:forest:real}, with these identified parameters, we can successfully perform the same tasks as the synthetic benchmarks on the real robot platform.

\begin{figure}[ht]
  \includegraphics[width=0.5\textwidth]{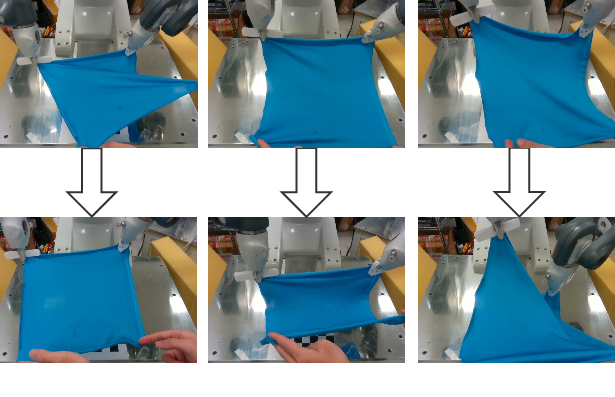}
\put(-220,3){(a)}
\put(-133,3){(b)}
\put(-45,3){(c)}
  \caption{ Real robot manipulation benchmarks: We highlight the real-time performance of random forest-based DOM controller on three tasks for robot and human collaboration (a) keep the cloth straight, (b) keep the cloth bent, (c) keep the cloth twisted.}
  \label{fig:forest:real}
\end{figure}

\subsubsection{Performance}

To evaluate the performance of each component in the random forest-based DOM controller, we run several variants of Algorithm~\ref{Alg:mainAlg}. All the meta-parameters used for training are illustrated in Table~\ref{table:param}. In our first set of experiments, we train a single-task random forest-based controller for each task. We profile the mean action error:
\begin{equation}
\label{eq:err}
\text{err}=\sum_{\langle \mathcal{O}(c),\mathbf r^*\rangle}\frac{1}{|\mathbf r^*||\mathcal{D}|}
\|\mathbf r^*-\frac{1}{K} \mathbf r_{l_k(\mathcal{O}(c)),k}^*\|^2,
\end{equation}
with respect to the number of IL iterations (Line~\ref{ln:outer} of Algorithm~\ref{Alg:mainAlg}). 

As illustrated in Figure~\ref{fig:forest:perf} (red), the action error reduces quickly within the first few iterations and converges later. We also plot the number of leaf-nodes in our random forest in Figure~\ref{fig:forest:perf} (green). As our controller converges, our random forest's topology also converges. We have also tried to infer the construction of the random forest (Line~\ref{ln:rs} of Algorithm~\ref{Alg:mainAlg}) from the IL algorithm, i.e. we construct a random forest during the first iteration and then keep it fixed, only learning the optimal control actions using IL. As illustrated in Figure~\ref{fig:forest:perf} (blue),
the residual in the precluded version is larger than that in the original version.
These experiments show that IL is suitable for training our random forest-based controller and allowing the random forest construction to be integrated into the IL algorithm does help improve controller performance. 
We also highlight the robustness of our method by comparing the performance of our controller with the basic controller using HOW features (Section~\ref{sec:featuredesign2}) in Figure~\ref{fig:forest:robust}.
This is done by applying both controllers to a scenario with randomized human actions.

\begin{figure}[ht]
  \centering
  \includegraphics[width=0.45\textwidth]{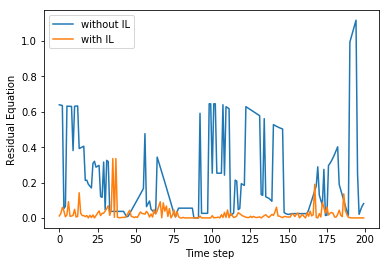}
  \caption{
    Robustness of the imitation learning algorithm:
  In a real-time human robot interaction, we plot the mean action error (Equation~\ref{eq:err}). The blue curve shows the performance of the controller trained using only one IL iteration and the orange curve shows the performance of the controller with 20 IL iterations. We compare residuals (Equation~\ref{eq:err}) between the two methods.
 The figure shows the robustness of the our method over the basic DOM controller using HOW features. }
  \label{fig:forest:robust}
\end{figure}

\begin{figure*}[t]
  \centering
  \includegraphics[width=0.32\textwidth]{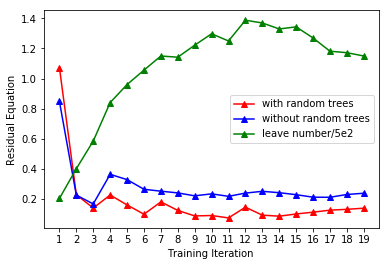}
  \includegraphics[width=0.32\textwidth]{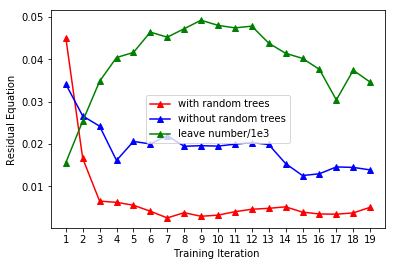}
  \includegraphics[width=0.32\textwidth]{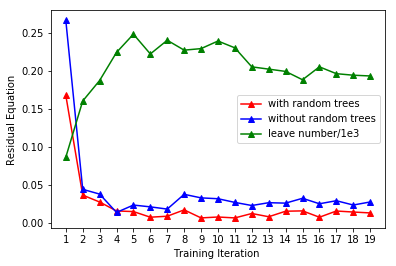}
  \put(-420,-8){(a)}
  \put(-240,-8){(b)}
  \put(-80,-8){(c)}
  \caption{
  Controller with random forest v.s. controller without random forest: (red): Residual (Equation~\ref{eq:err}) plotted against the IL iterations (Line~\ref{ln:outer} of Algorithm~\ref{Alg:mainAlg}). (green): Number of leaf-nodes plotted against the IL iterations. (blue): Residual Equation~\ref{eq:err} plotted against the IL iterations, which precludes random forest construction. (a): cloth straight; (b): cloth bent; (c): cloth twisted. As our controller converges, our random
forest’s topology also converges. Our random-forest-based controller outperforms the previous method. }
  \label{fig:forest:perf}
\end{figure*}

\begin{table}
\centering
\setlength{\tabcolsep}{10pt}
\begin{tabular}{ll}
Name & Value \\
\hline \\
Fraction term used in IL algorithm~\cite{Ross:2011:RIL} & $0.8$	\\
Training data collected in each IL iteration & $500$	\\
Resolution of RGB-D image & $640\times480$	\\
Dimension of HOW-feature used in Section~\ref{sec:featuredesign2} & $768$	\\
\TWORCell{Random forest's stopping criterion when}{impurity decrease less than \cite{Shotton:2011:RHP}} & $1 \times 10^{-4}$	\\
L1 regularization weights, $\theta$, in Equation~\ref{eq:opt_sparse} & $1 \times 10^{-5}$ \\
\hline \\
\end{tabular}
\caption{\label{table:param} Meta-parameters used for training.}
\end{table}

In our second experiment, we test the performance of 
the single-task controller, i.e., the controller that is trained for a single manipulation task such as bending. 
In our simulated environment, this can be performed by running our simulator.
During each simulation, we move the hands to $10$ random target positions $v^{2*},v^{3*}$.
As shown in Figure~\ref{fig:forest:multi}(red), we profile the residual (Equation~\ref{eq:err}) in all $10$ problems. Our controller performs consistently well with a relative action error of $0.4954$\%.

In the third experiment, we train a joint 3-task controller for all tasks. This is done by defining a single random forest and $3$ optimal actions on each leaf-node. The performance of the 3-task controller is compared with the single-task controller in Figure~\ref{fig:forest:multi}. The multi-task controller performs slightly worse in each task, but the difference is quite small. 
\begin{figure*}[t]
  \centering
  \includegraphics[width=0.32\textwidth]{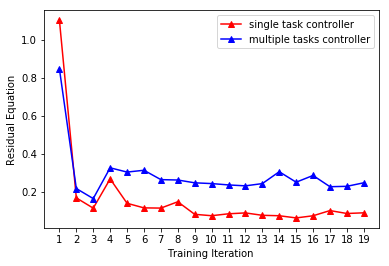}
  \includegraphics[width=0.32\textwidth]{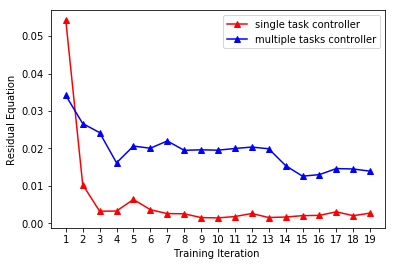}
  \includegraphics[width=0.32\textwidth]{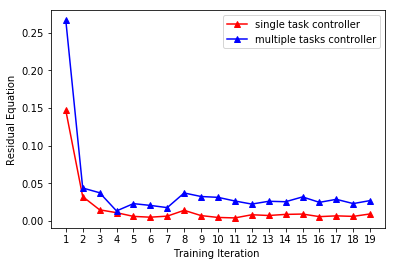}
  \put(-420,-8){(a)}
  \put(-240,-8){(b)}
  \put(-80,-8){(c)}
  \caption{
  Multi-task controller vs. single-task controller:
  Residual (Equation~\ref{eq:err}) in 10 random problems using joint 3-task controller (blue) and a single-task controller (red). (a) cloth straight; (b) cloth bent; (c) cloth twisted. Both controllers converge after a few iterations of the IL algorithm. The single-task controller performs consistently well with a relative action error of $0.4954$\%. The multi-task controller performs slightly worse in each task, but the difference is relatively small.}
  \label{fig:forest:multi}
\end{figure*}
\subsection{Benefits over prior approaches}
A key feature of our method is that it allows the robot to react to random human movements, while the effect of these movements is indirectly reflected via a piece of cloth. This setting is very similar to~\cite{Hirai:2000:ISP}. However, \cite{Hirai:2000:ISP} assumes the 3D cloth mesh $c$ is known without sensing error, which is not practical. 

Our method falls into a broader category of the visual-servoing methods, but most previous work in this area such as~\cite{Lee:2017:LVS} focused on navigation tasks and there is relatively little work on deformable body manipulation. \cite{bateux:2017:histograms} based their servoing engine on histogram features which is a similar method to our HOW-feature. However, they use direct optimization to minimize the cost function ($\mathbf{dist}(\mathcal{O}(c),\mathcal{O}(c^*))$) which is not possible in our case because our cost function is non-smooth in general.

Finally, our method is closely related to~\cite{Doumanoglou:2014:ARF,Doumanoglou:2014:AAR}, which also use random forest and store actions on the forest. However, our method is different in two ways. First, our controller is continuous in its parameters, which means it can be trained using an IL algorithm. We put both feature extraction and controller parametrization into the IL algorithm~\cite{Ross:2011:RIL} so that both the feature extractor and the controller benefit from evolving training data.

\section{Conclusion and future work}
\label{sec:conclusion}
We have presented a general approach framework to automatically servo-control soft objects using a dual-arm robot, and provide new solutions in both the feature design and the controller optimization. For feature design, we present a new algorithm to compute HOW-features, which capture the shape variation and local features of the deformable material using limited computational resources. The visual feedback dictionary is precomputed using sampling and clustering techniques and used with sparse representation to compute the velocity of the controller to perform the task. For controller design, we first propose an online GPR model to estimate the deformation function of the manipulated objects, and used low-dimension features to describe the object's configuration. The resulting GPR-based visual servoing system can generate high quality control velocities for the robotic end-effectors and is able to accomplish a set of manipulation tasks robustly. We also present a DOM controller where the optimal control action is defined on the leaf-nodes of a random forest. Further, both the random forest construction and controller optimization are integrated with an inmitaion learning algorithm and evolved with training data. We evaluate our proposed approaches on a set of DOM benchmark tasks. The result shows that our method can seamlessly combine feature extraction and controller parametrization problems. In addition, our method is robust to random noises in human motion and can accomplish high success rate.

Our approaches have some limitations and there are many avenues for future work.

For the HOW feature design, since HOW-features are computed from 2D images, the accuracy of the computations can also vary based on the illumination and relative colors of the cloth. For future work, we would like to make our approach robust to the training data and the variation of the environment.

For the GPR-based DOM controller, the deformation function actually depends not only on the feature velocity but also on the current state of the object in some complicated manipulation tasks. i.e., the input of the deformation function should add a variable to represent current state, like relative positions of feedback or manipulated points. However that function is relatively difficult to learn by online exploration since the model needs more data. For future work, we plan to find a better exploration method to learn a more complicate deformation function involving not only the feature velocity but also the current configuration of the object in the feature space, in order to achieve more challenging tasks like cloth folding.

For the random forest-based DOM controller, it is not easy to extend our method to reinforcement learning scenarios. This is because our method is not differentiable by involving a random forest construction. Therefore, reinforcement learning algorithms such as the policy gradient method~\cite{Peters:2006:PGM} cannot be used. Another potential drawback is that our method is still sensitive to the random forest's stopping criterion. When the stopping criterion is too loose, the number of leaf-nodes can increase considerably, leading to overfitting in Equation~\ref{eq:opt_sparse}. Finally, we need additional dimension reduction, i.e. the HOW-feature, and action labeling in the construction of the random forest. In our current version, labeling is done by mean-shift clustering of optimal actions, but in some applications, where observations can be semantically labeled, labeling observations instead of actions can be advantageous. For example, in object grasping tasks, we can construct our random forest to classify object types instead of classifying actions.

{\small
\bibliographystyle{IEEEtran}
\bibliography{references}
}

\end{document}